\documentclass[conference]{IEEEtran}
\IEEEoverridecommandlockouts

\usepackage{cite}
\usepackage{comment}
\usepackage{float}
\usepackage{amsmath,amssymb,amsfonts,amsthm, bbm}
\usepackage{algorithm}
\usepackage{algpseudocode}
\usepackage{graphicx}
\usepackage{subcaption}
\usepackage{textcomp}
\usepackage{xcolor}
\usepackage{tikz}
\usepackage{mathtools}
\usepackage[normalem]{ulem}
\usepackage[margin=0.75in]{geometry}
\usepackage{scalerel}
\usepackage{placeins}
\usepackage{overpic}
\usepackage[colorlinks=true]{hyperref}

\DeclareMathOperator*{\argmax}{\arg\,\max}
\DeclareMathOperator*{\argmin}{\arg\,\min}

\begin{document}

\newtheorem{hypotheses}{Hypotheses}
\newtheorem{problem}{Problem}
\newtheorem{example}{Example}
\newtheorem{definition}{Definition}
\newtheorem{assumption}{Assumption}
\newtheorem{theorem}{Theorem}
\newtheorem{lemma}{Lemma}
\newtheorem{corollary}{Corollary}[theorem]
\newtheorem{proposition}{Proposition}
\newtheorem*{remark}{Remark}
\newtheorem{conjecture}{Conjecture}
\numberwithin{assumption}{section}

\newcommand{\edit}[1]{\textcolor{blue}{#1}}
\newcommand{\ignore}[1]{}
\newcommand{\diff}{\mathop{}\!\mathrm{d}}

\algnewcommand{\Inputs}[1]{%
  \State \textbf{Inputs:}
  \Statex \hspace*{\algorithmicindent}\parbox[t]{.8\linewidth}{\raggedright #1}
}
\algnewcommand{\Initialize}[1]{%
  \State \textbf{Initialize:}
  \Statex \hspace*{\algorithmicindent}\parbox[t]{.8\linewidth}{\raggedright #1}
}

\title{
Collaborative Safety-Critical Formation Control with Obstacle Avoidance 
    \thanks{*Brooks A. Butler is with the Department of Electrical Engineering and Computer Science at the University of California, Irvine, Chi Ho Leung and~Philip. E. Par\'e are with the Elmore Family School of Electrical and Computer Engineering at Purdue University. Emails: bbutler2@uci.edu, leung61@purdue.edu and philpare@purdue.edu. This work was funded by Purdue’s Elmore Center for Uncrewed Aircraft Systems and the National Science Foundation, grant NSF-ECCS \#2238388.}
}

\author{Brooks A. Butler, Chi Ho Leung, and Philip E. Par\'{e}*}

\maketitle

\begin{abstract}
This work explores a collaborative method for ensuring safety in multi-agent formation control problems. We formulate a control barrier function (CBF) based safety filter control law for a generic distributed formation controller and extend our previously developed collaborative safety framework to an obstacle avoidance problem for agents with acceleration control inputs. We then incorporate multi-obstacle collision avoidance into the collaborative safety framework. This framework includes a method for computing the maximum capability of agents to satisfy their individual safety requirements. We analyze the convergence rate of our collaborative safety algorithm, and prove the linear-time convergence of cooperating agents to a jointly feasible safe action for all agents under the special case of a tree-structured communication network with a single obstacle for each agent. We illustrate the analytical results via simulation on a mass-spring kinematics-based formation controller and demonstrate the finite-time convergence of the collaborative safety algorithm in the simple proven case, the more general case of a fully-connected system with multiple static obstacles, and with dynamic obstacles.
\end{abstract}

\section{Introduction}

Multi-agent formation control problems have received special attention in robotics and automatic control due to their broad range of applications and theoretical challenges.
While it is impossible to categorize every formation control-related research exhaustively, we can organize them in terms of the fundamental ideas behind the control schemes \cite{beard2001coordination, reynolds1987flocks}, sensing capability and interaction topology of the formation controller \cite{oh2015survey}, and the formation control-induced problems of interest such as the consensus problem \cite{ren2010distributed}.
Some generalizations of the formation control-induced problems also find their application in other multi-agent cyber-physical systems outside the robotics community.
Some examples of critical multi-agent model applications include the mitigation of epidemic-spreading processes \cite{pare2020modeling, butler2023optimal}, smart grid management \cite{tuballa2016review}, and uncrewed aerial drone swarms \cite{tahir2019swarms}.
Since many of these multi-agent cyber-physical systems have become ubiquitous in modern society, effective and safe operation of multi-agent systems is crucial, as disruptions in these interconnected systems can potentially have far-reaching societal and economic consequences.

Theoretical frameworks and techniques from the study of safety-critical control offer promising solutions to the problem of collaborative safety requirements in the multi-agent formation problem.
Foundational work on safety-critical control can be traced back to the 1940s \cite{nagumo1942lage, blanchini1999set}.
Recently, the introduction and refinement of control barrier functions (CBFs) \cite{ames2016control, ames2019control} has greatly increased interest in the field of safety-critical control and its applications.
Since their introduction, control barrier
functions have been used in numerous applications to provide safety guarantees in various dynamic system models \cite{chung2018survey, wang2017safety}.
Moreover, multiple recent studies have reported CBFs' practicality and theoretical soundness in solving the multi-agent obstacle avoidance problem \cite{wang2017safety, santillo2021collision, jankovic2021collision}.

In large-scale multi-agent systems, using communication to coordinate actions between agents efficiently is a challenging problem. The field of cooperative control for multi-agent systems provides a rich body of literature that examines scenarios where agents may share information over a communication network \cite{lewis2013cooperative,li2017cooperative,wang2017cooperative,yu2017distributed}.  In such formulations, agents typically share and receive information via either direct communication or global broadcast that enables cooperative control adjustments to be made \cite{huang2010adaptive,li2021robust,qu2010cooperative,qu2012analytic}. However, in many formulations of cooperative control, a common assumption is that agent first-order dynamics are independent of each other, thus the networked element is only facilitated via virtual communication.
Further, for dynamically coupled systems, the influence of networked dynamics is often treated as bounded noise at the agent level~\cite{anand2022small}, Contrasting in our work, we wish to leverage any knowledge of the networked dynamic structure in the formulation of safety requests.

In this work, we make the following contributions:
%\vspace{-2ex}
\begin{enumerate}
    \item Under the formulation of a CBF-based safety-filter control law for a generic distributed formation controller, we extend our previously developed collaborative safety framework \cite{butler2023distributed} to obstacle avoidance for agents with acceleration control commands.
    \item We incorporate multi-obstacle collision avoidance into the collaborative safety framework that includes methods for computing the maximum capability of agents to satisfy their safety requirements.
    \item We prove the linear-time convergence of cooperating agents to a jointly feasible safe action for all agents under the special case of a tree-structured communication network for a single obstacle and demonstrate through simulation that the finite-time convergence rate of a fully-connected formation network with multiple/dynamic obstacles.
\end{enumerate}

A preliminary portion of this work was published and presented at the 2024 European Control Conference (ECC) \cite{butler2023collaborative}, where the major contributions, when compared with the previous version, include:
\begin{itemize}
    \item The proofs of all results omitted from \cite{butler2023collaborative} for space
    \item A detailed discussion and analysis of the convergence rate of the safety algorithm in Sections~\ref{sec:conflicting_safety} and \ref{sec:lin_alg_convergence}
    \item A full description of the collaborative safety algorithm and subroutines that handle multiple safety constraints
    \item Definition of an algorithm for selecting the closest point between two polytopic convex hulls in a way that encourages faster convergence of the algorithm
    \item New simulations that demonstrate the analytical results as well as the performance of our collaborative safety algorithm on a larger formation network (up to 8 agents) than was presented in \cite{butler2023collaborative} (3 agents).
\end{itemize}

\subsection{Notation}

Let $|\mathcal{C}|$ denote the cardinality of the set $\mathcal{C}$ and let $\partial \mathcal{C}$ denote the set of boundary points for a closed set $\mathcal{C}$. $\mathbb{R}$ and $\mathbb{N}$ are the sets of real numbers and positive integers, respectively. Let $C^r$ denote the set of functions $r$-times continuously differentiable in all arguments. We define ${\Vert \cdot \Vert}_2$ and ${\Vert \cdot \Vert}_{\infty}$ to be the two-norm and infinity norm of a given vector argument, respectively. We notate $\mathbf{0}$ and $\mathbf{1}$ to be vectors of all zeros and all ones, respectively, of the appropriate size given by context. For some vector $v$, we denote $v \geq \mathbf{0}$ and $v > \mathbf{0}$ to be the elementwise evaluation if all elements of $v$ are greater than or equal to zero, or strictly greater than zero, respectively and $[v]_k$ to be the $k$th element of vector $v$. A monotonically increasing continuous function $\alpha: \mathbb{R}_{+} \rightarrow \mathbb{R}_{+}$ with $\alpha(0) = 0$ is termed as class-$\mathcal{K}$. We define $[n] \subset \mathbb{N}$ to be a set of indices $\{1, 2, \dots, n\}$.
We define the Lie derivative of the function $h:\mathbb{R}^N \rightarrow \mathbb{R}$ with respect to the vector field generated by $f:\mathbb{R}^N \rightarrow \mathbb{R}^N$ as
\begin{equation}
    \mathcal{L}_f h(x) = \frac{\partial h(x)}{\partial x} f(x).
\end{equation}
We define high-order Lie derivatives with respect to the same vector field $f$ with a recursive formula \cite{robenack2008computation}, where $k>1$, as
\begin{equation}
    \mathcal{L}^k_f h(x) = \frac{\partial \mathcal{L}^{k-1}_f h(x)}{\partial x} f(x).
\end{equation}

\section{Preliminaries} \label{sec:preliminaries}
In this section, we provide some necessary background on preliminary concepts from the literature on high-order barrier functions and their application to networked dynamic systems.
We define a networked system using a graph $\mathcal{G} = (\mathcal{V}, \mathcal{E})$, where $\mathcal{V}$ is the set of $n = \vert \mathcal{V} \vert$ nodes, $ \mathcal{E} \subseteq \mathcal{V}\times \mathcal{V} $ is the set of edges. Let $\mathcal{N}_i$ be the set of all neighbors with an edge connection to node $i \in [n]$, where 
\begin{equation}
    \mathcal{N}_i = \{j \in [n]\setminus \{i\}: (i,j) \in \mathcal{E} \}.
\end{equation}
We further define $x_i$ to be the state vector for agent $i \in  [n]$, $x_{\mathcal{N}_i}$ to be the concatenated states of all neighbors to agent~$i$, i.e. $x_{\mathcal{N}_i} = \left( x_j, \forall j \in \mathcal{N}_i\right)$, and $x$ to be the full state of the networked system.

Recall the definition of \textit{high-order barrier functions} (HOBF) \cite{xiao2019control,xiao2021high}, where a series of functions are defined in the following general form
\begin{equation} \label{eq:HO_funcs}
    \begin{aligned}
        \psi_i^0(x) &:= h_i(x) \\
        \psi_i^1(x) &:= \dot{\psi}_i^0(x) + \alpha_i^{0}(\psi_i^0(x)) \\
        & \vdots \\
        \psi_i^k(x) &:= \dot{\psi}_i^{k-1}(x) + \alpha_i^{k-1}(\psi_i^{k-1}(x)),
    \end{aligned}
\end{equation}
where $k$ is the order of the barrier function $h_i \in C^r$, which is $r \geq k \geq 1$ times continuously differentiable, and $h_i:\mathbb{R}^{N_i} \rightarrow \mathbb{R}$ is a function whose zero-super-level set defines the region which node $i\in [n]$ considers to be safe and $\alpha_i^{0}(\cdot),\alpha_i^1(\cdot), \dots, \alpha_i^{k-1}(\cdot)$ are class-$\mathcal{K}$ functions of their argument. These functions provide definitions for the corresponding series of safety constraint sets
\begin{equation} \label{eq:HO_sets}
    \begin{aligned}
        \mathcal{C}_i^1 &:= \{ x \in \mathbb{R}^N: \psi_i^0(x) \geq 0 \} \\
        \mathcal{C}_i^2 &:= \{ x \in \mathbb{R}^N: \psi_i^1(x) \geq 0 \} \\
        & \vdots \\
        \mathcal{C}_i^k &:= \{ x \in \mathbb{R}^N: \psi_i^{k-1}(x) \geq 0 \}. \\
    \end{aligned}
\end{equation}

%\vspace{-3ex}

Typically, HOBFs are used in the context of systems where the control input is applied in at least the second-order dynamics. For example, in many robotics systems, control is implemented through acceleration inputs and safety is defined according to relative position \cite{breeden2022compositions}. However, in the context of networked dynamic systems, we can use concepts from HOBFs to define a barrier function that encodes the dynamic effects of neighbor nodes as follows.
%\vspace{1ex}
\begin{definition}
    Let $\mathcal{C}_i^1, \mathcal{C}_i^2, \dots, \mathcal{C}_i^k$ be defined by \eqref{eq:HO_funcs} and \eqref{eq:HO_sets}. We have that $h_i$ is a \textbf{k$^{th}$-order node-level barrier function} (NBF) for node $i\in [n]$ if $h_i \in C^k$ and there exist differentiable class-$\mathcal{K}$ functions $\alpha_i^{0},\alpha_i^1,\dots, \alpha_i^{k-1}$ such that $\psi_i^k(x) \geq 0$ for all $x \in \bigcap_{r=1}^k \mathcal{C}_i^r$. 
\end{definition}
\noindent
This definition leads naturally to the following lemma, which is a direct result of \cite[Theorem 4]{xiao2019control}.
\begin{lemma} \label{lem:NBF}
If $h_i$ is an NBF, then $\bigcap_{r=1}^k \mathcal{C}_i^r$ is forward invariant.
\end{lemma}
\noindent
In this sense, under Lemma~\ref{lem:NBF}, we may consider $\bigcap_{r=1}^k \mathcal{C}_i^r$ to be a node-level viability domain \cite{gurriet2020scalable} of $\mathcal{C}_i$ with respect to the $(k-1)$-hop neighborhood dynamics of node $i \in [n]$. In the following section, we present a formulation for the safety control problem and show how agents can use the higher-order dynamics of a formation control law to communicate safety needs in a cooperative system.

\section{Problem Formulation} \label{sec:problem_formulation}

Consider the first-order dynamics for a single agent~$i$
\begin{equation} \label{eq:dynamics}
    \dot{x}_i = f_i(x_i) + g_i(x_i)u_i,
\end{equation}
where $u_i \in \mathcal{U}_i \subset \mathbb{R}^{M_i}$ is some form of affine acceleration controller for agent~$i$.  
Let $u_i^f(x_i, x_{\mathcal{N}_i})$ be a distributed feedback control law that induces some formation behavior. We can treat these formation dynamics  as part of the natural dynamics of the system where $u_i^f(x_i, x_{\mathcal{N}_i})$ is modified by some safety-filtered control law as

%\vspace{-4ex}

\begin{equation*}
    \dot{x}_i = f_i(x_i) + g_i(x_i)(u_i^f(x_i, x_{\mathcal{N}_i})-u_i^s),
\end{equation*}
where $u_i^s$ is a modification to the formation control signal to ensure agent safety. We can then rewrite the dynamics in \eqref{eq:dynamics} as
\begin{equation} \label{eq:formation_dynamics}
    \dot{x}_i = \bar{f}_i(x_i, x_{\mathcal{N}_i}) + \bar{g}_i(x_i)u_i^s
\end{equation}
where
\begin{equation}
    \bar{f}_i(x_i, x_{\mathcal{N}_i}) = f_i(x_i) + g_i(x_i)u_i^f(x_i, x_{\mathcal{N}_i})
\end{equation}
and
\begin{equation}
    \bar{g}_i(x_i) = -g_i(x_i).
\end{equation}

%\vspace{-2ex}

We assume each agent has positional safety constraints with respect to obstacles $o \in \mathcal{O}_i(t)$, where $\mathcal{O}_i(t)$ is the set of identifiers for obstacles within the sensing range of agent~$i$ at time $t$. 
For convenience, we drop the notation of time dependence on $\mathcal{O}_i$ moving forward. We define the set of viable safety-filtered control actions as
\begin{equation} \label{eq:constraints_for_i}
    \mathcal{U}_i^s(x_i, x_{\mathcal{N}_i}) = \{ u_i^s \in \mathcal{U}_i: u_i^f(x_i, x_{\mathcal{N}_i}) - u_i^s \in \mathcal{U}_i \}.
\end{equation}
In this paper, we assume safety conditions for each agent are defined with respect to the relative position of agents to obstacles. Therefore, since control is implemented through acceleration, we construct a high-order barrier function for each agent~$i$ with respect to a given obstacle $o$ as follows
\begin{equation}\label{eq:second_order_BF}
    \begin{aligned}
        \phi_{i,o}^0(x_i, x_o) &= h_i(x_i, x_o) \\
        \phi_{i,o}^1(x_i, x_o) &= \dot{\phi}_{i,o}^0(x_i, x_o) + \alpha_i^0(\phi_{i,o}^0(x_i, x_o)),
    \end{aligned}
\end{equation}
where $x_o$ is the state of obstacle $o \in \mathcal{O}_i$. These functions then define the corresponding safety constraint sets
\begin{equation} \label{eq:HO_obst_sets}
    \begin{aligned}
        \mathcal{C}_{i,o}^1 &:= \{ (x_i, x_o) \in \mathbb{R}^{N_i} \times \mathbb{R}^{N_o}: \phi_{i,o}^0(x_i, x_o) \geq 0 \} \\
        \mathcal{C}_{i,o}^2 &:= \{ (x_i, x_o) \in \mathbb{R}^{N_i} \times \mathbb{R}^{N_o}: \phi_{i,o}^1(x_i, x_o) \geq 0 \}.
    \end{aligned}
\end{equation}

%\vspace{-3ex}

\noindent Given the definition of these constraint sets, we can define an \textit{agent-level control barrier function} and subsequent forward invariant properties as follows. For the sake of notational brevity, we use $\mathbf{x}_i$ to denote the concatenated states of agents in the neighborhood centered on agent $i \in [n]$, $(x_i, x_{\mathcal{N}_i})$, moving forward.

%\vspace{1ex}
\begin{definition}
    We have $h_{i,o}(x_i, x_o)$ is an \textbf{agent-level control barrier function} (ACBF) if for all $(x_i, x_o) \in \mathcal{C}_{i,o}^1 \cap \mathcal{C}_{i,o}^2$ and $t \in \mathcal{T}$ there exists a class-$\mathcal{K}$ functions $\alpha_i^0$ and $\alpha_i^1$ and $u_i^s \in \mathcal{U}_i^s(\mathbf{x}_i)$ such that
    \begin{equation}\label{eq:first_order_safety_cond}
        \dot{\phi}_{i,o}^1(\mathbf{x}_i, x_o, u_i^f(\mathbf{x}_i), u_i^s) + \alpha_i^1(\phi_{i,o}^1(x_i, x_o)) \geq 0.
    \end{equation}
\end{definition}

%\vspace{-3ex}

We see that \eqref{eq:first_order_safety_cond} characterizes the first-order safety condition for agent~$i$ with respect to obstacle $o$ since the acceleration control input appears in the second derivative of $h_{i,o}$, which is computed in $\dot{\phi}_{i,o}^1$. This barrier function definition naturally leads to the following result on agent-level safety.

%\vspace{1ex}
\begin{lemma}\label{lem:ACBF}
    Given a distributed multi-agent system defined by \eqref{eq:formation_dynamics} and constraint sets defined by \eqref{eq:second_order_BF} and \eqref{eq:HO_obst_sets},
    if $h_{i,o}(x_i, x_o)$ is an ACBF, then $\mathcal{C}_{i,o}^1 \cap \mathcal{C}_{i,o}^2$ is forward invariant for all $t \in \mathcal{T}$.
\end{lemma}
%\vspace{-3.5ex}
\begin{proof}
    If $h_{i,o}$ is an ACBF, then $\exists u_i^s \in \mathcal{U}_i^s$ such that
    %\vspace{-1ex}
    \begin{equation*}
        \dot{\phi}_{i,o}^1(\mathbf{x}_i, x_o, u_i^f(\mathbf{x}_i), u_i^s) + \alpha_i^1(\phi_{i,o}^1(x_i, x_o)) \geq 0,
    \end{equation*}
    for all $(x_i, x_o) \in \mathcal{C}_{i,o}^1 \cap \mathcal{C}_{i,o}^2$. Thus, as $\phi_{i,o}^1(x_i, x_o)$ approaches zero, there will be some $u_i^s$ such that $\dot{\phi}_{i,o}^1(\mathbf{x}_i, x_o, u_i^f(\mathbf{x}_i), u_i^s) \geq 0$. Therefore, if $(x_i(t_0), x_o(t_0)) \in \mathcal{C}_{i,o}^2$ then $\mathcal{C}_{i,o}^2$ is forward invariant for all $t \in \mathcal{T}$. By Lemma~\ref{lem:NBF}, it follows that is $\mathcal{C}_{i,o}^1 \cap \mathcal{C}_{i,o}^2$ forward invariant.
\end{proof}

%\vspace{-3ex}

% \noindent
With agent-level control barrier functions defined, we are now prepared to describe our main objective for this work: 
\begin{equation} \label{eq:problem_statement}
    \begin{aligned}
        \min_{u_i^s \in \mathcal{U}_i^s(x)} \quad & \frac{1}{2}{\left\Vert u_i^f(x_i, x_{\mathcal{N}_i}) - u_i^s \right\Vert}_2^2 \\
        \text{s.t.} \quad & \dot{\phi}_{i,o}^1\left(\mathbf{x}_i, x_o, u_i^f(\mathbf{x}_i), u_i^s \right) + \alpha_i^1\left(\phi_{i,o}^1(x_i, x_o)\right) \geq 0 \\
        & \forall i \in [n], \; \forall o \in \mathcal{O}_i.
    \end{aligned}
\end{equation}
\noindent
In other words, we aim to provide a control policy that minimally alters the prescribed distributed formation control signal such that the defined safety conditions for obstacle avoidance are satisfied for all agents in the formation.

\section{Collaborative Safe Formation Control}\label{sec:safe_with_colab}
We now present a method by which each agent can communicate safety needs to its neighboring agents to achieve collective safety in a distributed manner.
We define a relative position safety constraint for each agent with respect to a given obstacle as follows In this work, we consider position-based safety conditions where the full state of each agent includes both position and velocity, i.e., $x_i = [p_i^\top, v_i^\top]^\top$, with $p_i$ and $v_i$ being of the proper spatial dimension dependent on the application. Let $p_o$ be the position of obstacle $o \in \mathcal{O}_i$, where $\mathcal{O}_i$ is the set of all obstacles for agent $i \in [n]$. We define a position-based safety constraint as 
\begin{equation} \label{eq:rel_dist_barrierfunc}
    h_{i,o}(x_i, x_o) = {\Vert p_i - p_o \Vert}^2_2 - r_{i,o}^2
\end{equation}
where $r_{i,o} \in \mathbb{R}$ is the minimum distance agent~$i$ should maintain from obstacle $o$. Assuming control inputs on the acceleration of agent~$i$, we use the second-order barrier functions candidate from \eqref{eq:second_order_BF} to define the first-derivative safety condition

%\vspace{-5ex}

\begin{equation}
    \dot{\phi}_{i,o}^1(\mathbf{x}_i, x_o, u_i^s) = \mathcal{L}_{\bar{f}_i}\phi_{i,o}^1(\mathbf{x}_i, x_o) + \mathcal{L}_{\bar{g}_i}\phi_{i,o}^1(x_i, x_o) u_i^s.
\end{equation}

%\vspace{-3ex}

We can define the next high-order barrier function that captures the coupling behavior of the formation as

%\vspace{-4ex}

\footnotesize
\begin{subequations} \label{eq:third_order_BF}
    \begin{align}
        \phi_{i,o}^2(\mathbf{x}_i, x_o, u_i^s) &= \dot{\phi}_{i,o}^1(\mathbf{x}_i, x_o, u_i^s) + \alpha_i^1(\phi_{i,o}^1(x_i, x_o)) \label{eq:phi2} \\
        \phi_{i,o}^3(\mathbf{x}_i, x_o, u_i^s, \dot{u}_i^s, u_{\mathcal{N}_i}^s) &= \dot{\phi}_{i,o}^2(\mathbf{x}_i, x_o, u_i^s, \dot{u}_i^s, u_{\mathcal{N}_i}^s) \label{eq:phi3}
        \nonumber  
        \\ & \quad 
        + \alpha_i^2(\phi_{i,o}^2(\mathbf{x}_i, x_o, u_i^s)), 
    \end{align}
\end{subequations}
\normalsize

%\vspace{-2ex}

similarly to \eqref{eq:second_order_BF}, where $\alpha_i^2(\cdot)$ is a class-$\mathcal{K}$ function. 
Notice that the dynamics of each neighbor $j \in \mathcal{N}_i$ 
and $\dot{u}_i^s$ appear in \eqref{eq:phi3}. To assist in our analysis of the high-order dynamics of \eqref{eq:phi3}, we make the following assumption.
%\vspace{1ex}
\begin{assumption} \label{assume:u_dot_func_u}
    For a given node~$i\in [n]$, let $\dot{u}_i^s := d(u_i^s)$, where $d(u_i^s): \mathbb{R}^{M_i} \rightarrow \mathbb{R}^{M_i}$ is locally Lipschitz.
\end{assumption}

%\vspace{-2ex}
While obtaining a closed-form solution for $\dot{u}_i$ may be challenging in some applications, in practice, $d(u_i)$ may be approximated using discrete-time methods.
A more detailed discussion on the derivation of \eqref{eq:third_order_BF} may be found in \cite{butler2023distributed}; however, for our purposes, we separate \eqref{eq:phi3} into terms that are affected by neighbors' control and those that are not affected by neighbors' control, under Assumption~\ref{assume:u_dot_func_u}, as follows 

%\vspace{-5ex}

\small
\begin{equation} \label{eq:phi3_separated}
    \phi_{i,o}^3(\mathbf{x}_i, x_o, u_i^s, u_{\mathcal{N}_i}^s) = \sum_{j \in \mathcal{N}_i} a_{ij,o}(\mathbf{x}_i, x_o)u_j^s + c_{i,o}(\mathbf{x}_i, x_o, u_i^s),
\end{equation}
\normalsize
where
\begin{equation} \label{eq:neighbor_effects}
    a_{ij,o}(\mathbf{x}_i, x_o) = \mathcal{L}_{\bar{g}_j}\mathcal{L}_{\bar{f}_i} \phi_{i,o}^1(\mathbf{x}_i, x_o)
\end{equation}
is the effect that modified control actions $u_j^s$ taken by agent $j \in \mathcal{N}_i$ have on the formation dynamics and the subsequent safety condition of agent~$i$ with respect to obstacle $o \in \mathcal{O}_i$ and $c_{i,o}(\mathbf{x}_i, x_o, u_i^s)$ collects all other terms including those that are affected by its own control actions $u_i^s$. 
To compute $c_{i,o}$ more explicitly, we make the following assumption,
%\vspace{1ex}
\begin{assumption} \label{assume:class_K_scalar}
    Let $\alpha_i^1(z) := \alpha_i^1 z$ and $\alpha_i^2(z) := \alpha_i^2 z$, where $z \in \mathbb{R}^{N_i}$ and $\alpha_i^1, \alpha_i^2 \in \mathbb{R}_{> 0}$. 
\end{assumption}

If we define $\beta_i = \alpha_i^1 +  \alpha_i^2$, then by the derivation of \eqref{eq:phi3} from \cite{butler2023distributed}, under Assumptions~\ref{assume:u_dot_func_u} and \ref{assume:class_K_scalar} we derive the full expression of $c_{i,o}$ as

\begin{equation} \label{eq:safety_capability}
    % \small
    \begin{aligned}
        c_{i,o}(\mathbf{x}_i, x_o, u_i^s) &= \sum_{j \in \mathcal{N}_i} \mathcal{L}_{\bar{f}_j}\mathcal{L}_{\bar{f}_i} \phi_{i,o}^1 + \mathcal{L}_{\bar{f}_i}^2 \phi_{i,o}^1 \\
        & \quad + \alpha_i^1 \alpha_i^2 \phi_{i,o}^1 + \beta_i \mathcal{L}_{\bar{f}_i}\phi_{i,o}^1 + \mathcal{L}_{\bar{g}_i}\phi_{i,o}^1 d(u_i^s)\\
        & \quad + u_i^{s\top}\mathcal{L}_{\bar{g}_i}^2 \phi_{i,o}^1 u_i^s + \beta_i \mathcal{L}_{\bar{g}_i} \phi_{i,o}^1 u_i^s \\
        & \quad + \left[ \mathcal{L}_{\bar{f}_i} \mathcal{L}_{\bar{g}_i} \phi_{i,o}^{1\top} +\mathcal{L}_{\bar{g}_i} \mathcal{L}_{\bar{f}_i} \phi_{i,o}^1 \right]u_i^s
    \end{aligned}
\end{equation}

%\vspace{-3ex}

\noindent
which we use to define the safety capability for agent~$i$ as follows.
%\vspace{1ex}
\begin{definition}\label{def:safe_cap}
    Under Assumption~\ref{assume:u_dot_func_u}, the total \textbf{safety capability}, $c_{i,o}(\mathbf{x}_i, x_o, u_i^s)$, of agent~$i$ for a given action $u_i^s \in \mathbb{R}^{M_i}$ with respect to obstacle $o \in \mathcal{O}_i$ is computed by \eqref{eq:safety_capability}, where $c_{i,o}(\mathbf{x}_i, x_o, u_i^s) \geq 0$ indicates that agent~$i$ is capable of remaining safe, assuming no adverse effects from neighbors (i.e., $a_{ij,o}(\mathbf{x}_i, x_o) u_j^s \geq 0, \forall j \in \mathcal{N}_i$). Conversely, $c_{i,o}(\mathbf{x}_i, x_o, u_i^s) < 0$ indicates a deficit in agent~$i$'s capability to meet its safety requirement.  
\end{definition}

Given our definition of a subsequent higher-order barrier function in \eqref{eq:third_order_BF}, we define another safety constraint set as
\begin{equation} \label{eq:HO_neighbor_set}
    \begin{aligned}
        \mathcal{C}_{i,o}^3 := & \big\{ (x_i, x_o) \in \mathbb{R}^{N_i} \times \mathbb{R}^{N_o}: \exists u_i^s \in \mathcal{U}_i^s 
        \\ & \quad 
        \text{ s.t. } \phi_{i,o}^2(\mathbf{x}_i, x_o, u_i^f, u_i^s) \geq 0 \big\},
    \end{aligned}
\end{equation}
which collects all states where agent~$i$ is capable of maintaining its first-order safety condition under the influence of its induced formation dynamics. Given these definitions, we are prepared to define a collaborative control barrier function as follows.
%\vspace{1ex}
\begin{definition}
    Let $\mathcal{C}_{i,o}^1$, $\mathcal{C}_{i,o}^2$, and  $\mathcal{C}_{i,o}^3$ be defined by \eqref{eq:HO_obst_sets} and \eqref{eq:HO_neighbor_set}. We have that $h_{i,o}$ is a \textbf{collaborative control barrier function} (CCBF) for node $i \in [n]$ if $h_{i,o} \in C^3$ and $\forall (x_i,x_o) \in \mathcal{C}_{i,o}^1 \cap \mathcal{C}_{i,o}^2 \cap \mathcal{C}_{i,o}^3$ and $\forall t \in \mathcal{T}$ there exists $(u_i^s, u_{\mathcal{N}_i}^s) \in \mathcal{U}_i^s \times \mathcal{U}_{\mathcal{N}_i}^s$ such that, under Assumption~\ref{assume:u_dot_func_u}, 
    \begin{equation} \label{eq:CCBF_cond}
        \phi_{i,o}^3(\mathbf{x}_i, x_o, u_i^s, u_{\mathcal{N}_i}^s) \geq 0, \; \forall o \in \mathcal{O}_i.
    \end{equation}
\end{definition}

\begin{lemma} \label{lem:CCBF}
    Given a distributed multi-agent system defined by \eqref{eq:formation_dynamics} and constraint sets defined by \eqref{eq:second_order_BF}, \eqref{eq:HO_obst_sets}, \eqref{eq:third_order_BF} and \eqref{eq:HO_neighbor_set}, if $h_{i,o}$ is a CCBF for all $o \in \mathcal{O}_i$, then $\bigcap_{o \in \mathcal{O}_i}\mathcal{C}_{i,o}^1 \cap \mathcal{C}_{i,o}^2 \cap \mathcal{C}_{i,o}^3$ is forward invariant $\forall t \in \mathcal{T}$. 
\end{lemma}
%\vspace{-2.5ex}
\begin{proof}
    The result of this lemma is a direct extension of Theorem 2 in \cite{butler2023distributed}, where if $h_{i,o}$ is a CCBF for a given obstacle $o \in \mathcal{O}_i$ then $\exists (u_i^s, u_{\mathcal{N}_i}^s) \in \mathcal{U}_i^s \times \mathcal{U}_{\mathcal{N}_i}^s$ such that \eqref{eq:CCBF_cond} holds.
    Since $u_i^s$ appears in both $\phi_{i,o}^3(\mathbf{x}_i, x_o, u_i^s, u_{\mathcal{N}_i}^s)$ and $\phi_{i,o}^2(\mathbf{x}_i, x_o, u_i^s)$, we must show that if $(x_i,x_o) \in \mathcal{C}_{i,o}^1 \cap \mathcal{C}_{i,o}^2 \cap \mathcal{C}_{i,o}^3$ and $\phi_{i,o}^3(\mathbf{x}_i, x_o, u_i^s, u_{\mathcal{N}_i}^s) \geq 0$ for some $u_i^s \in \mathcal{U}_i^s$, then $\phi_{i,o}^2(\mathbf{x}_i, x_o, u_i^s) \geq 0$ also. If \eqref{eq:CCBF_cond} holds for all $(x_i,x_o) \in \mathcal{C}_{i,o}^1 \cap \mathcal{C}_{i,o}^2 \cap \mathcal{C}_{i,o}^3$, then for all $x_i, x_{\mathcal{N}_i}, x_o$ and $u_i^s \in \mathcal{U}_i^s$ where $\phi_{i,o}^2(\mathbf{x}_i, x_o, u_i^s) = 0$, there exists $u_{\mathcal{N}_i}^s \in \mathcal{U}_{\mathcal{N}_i}^s$ such that $\dot{\phi}_{i,o}^2(\mathbf{x}_i, x_o, u_i^s, u_{\mathcal{N}_i}^s) \geq 0$. Thus, we have that $\phi_{i,o}^2(\mathbf{x}_i, x_o, u_i^s) \geq 0, \forall (x_i,x_o) \in \mathcal{C}_{i,o}^1 \cap \mathcal{C}_{i,o}^2 \cap \mathcal{C}_{i,o}^3$, which implies $\phi_i^1(x_i, x_o) \geq 0, \forall (x_i,x_o) \in \mathcal{C}_{i,o}^1 \cap \mathcal{C}_{i,o}^2 \cap \mathcal{C}_{i,o}^3$. Therefore, we have that $\mathcal{C}_{i,o}^1 \cap \mathcal{C}_{i,o}^2 \cap \mathcal{C}_{i,o}^3$ is forward invariant. Further, since these same arguments hold $\forall o \in \mathcal{O}_i$, 
    it directly follows that $\bigcap_{o \in \mathcal{O}_i}\mathcal{C}_{i,o}^1 \cap \mathcal{C}_{i,o}^2 \cap \mathcal{C}_{i,o}^3$ is forward invariant.
\end{proof}

%\vspace{-2ex}

With set invariance defined with respect to neighbors' influence, we can leverage these properties to construct an algorithm to implements collaborative safety through rounds of communication between neighbors.

\subsection{Collaboration Through Communication} \label{sec:colab_algo}
In this section, we introduce the \textit{collaborative safety algorithm}, modified from our previous work in \cite{butler2023distributed}. The major additional contribution to the algorithm in this work is the additional handling of multiple safety constraints from each agent, which requires a new definition of maximum safety capability with respect to multiple safety conditions. 
For the formation control problem scenario, we make the following assumption.

\begin{assumption} \label{assume:Lgi^2_is_0}
    Let $\mathcal{L}_{\bar{g}_i}^2 \phi_{i,o}^1(x_i, x_o) = \mathbf{0}^{M_i \times M_i}, \forall o \in \mathcal{O}_i$.
\end{assumption}
%\vspace{-2ex}
In words, we assume that the control exerted by agent~$i$ does not have a dynamic relationship with its ability to exert control (e.g., the robot's movement is implemented identically no matter its position in a defined coordinate system). Since each agent may be actively avoiding multiple obstacles, we compute the stacked vector describing \eqref{eq:phi3_separated} for all obstacles $o \in \mathcal{O}_i$ under Assumptions~\ref{assume:u_dot_func_u}-\ref{assume:Lgi^2_is_0}, as follows

%\vspace{-5ex}

\small
\begin{equation} \label{eq:capability_vector_expanded}
    \begin{aligned}
        &\Phi_i =
        \underbrace{
        \begin{bmatrix}
            a_{i1,1} & \cdots & a_{i{|\mathcal{N}_i|},1} \\
            \vdots & \ddots & \vdots \\
            a_{i1,{K}} & \cdots & a_{i{|\mathcal{N}_i|},{K}}
        \end{bmatrix}
        }_{A_i}
        \begin{bmatrix}
            u_{1}^s \\ \vdots \\ u_{{|\mathcal{N}_i|}}^s
        \end{bmatrix}
        +
        \underbrace{
        \begin{bmatrix}
            \mathcal{L}_{\bar{g}_i} \phi_{i,{1}}^1 \\
            \vdots \\
            \mathcal{L}_{\bar{g}_i} \phi_{i,{K}}^1
        \end{bmatrix}
        }_{D_i}
        d(u_i^s)
        \\
        &+
        \underbrace{
        \begin{bmatrix}
            \mathcal{L}_{\bar{f}_i} \mathcal{L}_{\bar{g}_i} \phi_{i,{1}}^{1\top} + \mathcal{L}_{\bar{g}_i} \mathcal{L}_{\bar{f}_i} \phi_{i,{1}}^1 + \beta_i \mathcal{L}_{\bar{g}_i} \phi_{i,{1}}^1 \\
            \vdots \\
            \mathcal{L}_{\bar{f}_i} \mathcal{L}_{\bar{g}_i} \phi_{i,{K}}^{1\top} + \mathcal{L}_{\bar{g}_i} \mathcal{L}_{\bar{f}_i} \phi_{i,{K}}^1 + \beta_i \mathcal{L}_{\bar{g}_i} \phi_{i,{K}}^1
        \end{bmatrix}
        }_{B_i}
        u_i^s 
        \\
        &+
        \underbrace{
        \begin{bmatrix}
            \sum_{j \in \mathcal{N}_i} \mathcal{L}_{\bar{f}_j}\mathcal{L}_{\bar{f}_i} \phi_{i,{1}}^1 + \mathcal{L}_{\bar{f}_i}^2 \phi_{i,{1}}^1 + \alpha_i^1 \alpha_i^2 \phi_{i,{1}}^1 + \beta_i \mathcal{L}_{\bar{f}_i}\phi_{i,{1}}^1 \\
            \vdots \\
            \sum_{j \in \mathcal{N}_i} \mathcal{L}_{\bar{f}_j}\mathcal{L}_{\bar{f}_i} \phi_{i,{K}}^1 + \mathcal{L}_{\bar{f}_i}^2 \phi_{i,{K}}^1 + \alpha_i^1 \alpha_i^2 \phi_{i,{K}}^1 + \beta_i \mathcal{L}_{\bar{f}_i}\phi_{i,{K}}^1
        \end{bmatrix}
        }_{q_i},
    \end{aligned}
\end{equation}
\normalsize
where $K = |\mathcal{O}_i(t)|$ is the number of obstacles in $\mathcal{O}_i$ within the sensing range of agent~$i$ at time $t$. 
Note that the length of this vector is time-varying according to $|\mathcal{O}_i(t)|$. For convenience, we define the tuple containing all relevant state information of obstacles with respect to agent~$i$ as $\mathbf{x}_{\mathcal{O}_i} = (x_{o_1}, \dots, x_{o_{|\mathcal{O}_i|}})$.
We can therefore express \eqref{eq:capability_vector_expanded} more compactly as
\begin{equation} \label{eq:multi-obs-compact}
    \Phi_i(\mathbf{x}_i,\mathbf{x}_{\mathcal{O}_i}, u_i^s, u_{\mathcal{N}_i}^s) = A_i u_{\mathcal{N}_i}^s + D_i d(u_i^s) + B_i u_i^s + q_i,
\end{equation}
where $A_i \in \mathbb{R}^{K \times M_{\mathcal{N}_i}}$, with $M_{\mathcal{N}_i} = \sum_{j \in \mathcal{N}_i} M_j$, $D_i,B_i \in \mathbb{R}^{K \times M_i}$, and $q_i \in \mathbb{R}^{K}$. Computing the safety condition for each obstacle using \eqref{eq:multi-obs-compact}, we may interpret $A_i$ as the matrix describing the effect of each neighbor's control input, $B_i$ and $D_i$ are the matrices describing the effect of agent~$i$'s control input, and $q_i$ is a vector that collects all uncontrolled terms. 
We also introduce notation for selecting block columns of $A_i$ that isolate the dynamic relationship between the control input of agent $j \in \mathcal{N}_i$, $u_j \in \mathbb{R}^{M_j}$, and the safety of agent~$i$ with respect to each obstacle in $\mathcal{O}_i$ with the matrix 
\begin{equation} \label{eq:block_col_A_i}
    A_{ij,\mathcal{O}_i} = [a_{ij,1}^\top, \, \dots \, , a_{ij,{K}}^\top]^\top,
\end{equation}
where $ A_{ij,\mathcal{O}_i} \in \mathbb{R}^{K \times M_j}$, which is used in Algorithm~\ref{alg:coordinate} in Section~\ref{sec:conflicting_safety} to locally coordinate multiple safety conditions from each neighbor simultaneously.
We have the following result on its relationship of \eqref{eq:capability_vector_expanded} to the problem stated in \eqref{eq:problem_statement}.
%\vspace{1ex}
\begin{lemma} \label{lem:equivalent_solution}
    Under Assumptions~\ref{assume:u_dot_func_u}-\ref{assume:Lgi^2_is_0}, any control inputs $(u_i^s, u_{\mathcal{N}_i}^s) \in \mathcal{U}_i^s \times \mathcal{U}_{\mathcal{N}_i}^s, \forall i \in [n]$ that satisfy
    \begin{equation} \label{eq:lem_compact_condition}
       \Phi_i(\mathbf{x}_i, \mathbf{x}_{\mathcal{O}_i}, u_i^s, u_{\mathcal{N}_i}^s) \geq \mathbf{0}, \forall i \in [n]
    \end{equation}
    from \eqref{eq:multi-obs-compact} are also a solution to 
    \begin{equation}\label{eq:lem_prob_statement}
        \dot{\phi}_{i,o}^1 (\mathbf{x}_i, x_o, u_i^f, u_i^s) + \alpha_i^1\left(\phi_{i,o}^1(x_i, x_o)\right) \geq 0, \forall i \in [n], \forall o \in \mathcal{O}_i
    \end{equation}
    from \eqref{eq:problem_statement}.
\end{lemma}
%\vspace{-3ex}
\begin{proof}
    By the proof of Lemma~\ref{lem:CCBF}, if $\phi_{i,o}^3(\mathbf{x}_i, x_o, u_i^s, u_{\mathcal{N}_i}^s) \geq 0$ for some $u_i^s \in \mathcal{U}_i^s$, under Assumption~\ref{assume:u_dot_func_u}, then 
    
    %\vspace{-5ex}
    \small
    $$\phi_{i,o}^2(\mathbf{x}_i, x_o, u_i^s) = \dot{\phi}_{i,o}^1 (\mathbf{x}_i, x_o, u_i^f, u_i^s) + \alpha_i^1\left(\phi_{i,o}^1(x_i, x_o)\right) \geq 0$$
    \normalsize
    
    %\vspace{-3ex}
    also. Thus, since \eqref{eq:lem_compact_condition} implies that $\phi_{i,o}^3(\mathbf{x}_i, x_o, u_i^s, u_{\mathcal{N}_i}^s) \geq 0, \forall o \in \mathcal{O}_i$, then we can simplify the expression of \eqref{eq:phi3_separated} by selecting scalar class-$\mathcal{K}$ functions by Assumption~\ref{assume:class_K_scalar}, 
    % setting $\mathcal{L}_{\bar{g}_i}\phi_{i,o}^1(x_i, x_o)\dot{u}_i^s = 0$ by Assumption~\ref{assume:u_i_piecewise_constant}, 
    and setting $u_i^{s\top}\mathcal{L}_{\bar{g}_i}^2 \phi_{i,o}^1(\mathbf{x}_i, x_o) u_i^s = 0, \forall o\in \mathcal{O}_i$ by Assumption~\ref{assume:Lgi^2_is_0}, the set of control inputs that satisfy \eqref{eq:lem_compact_condition} must also satisfy \eqref{eq:lem_prob_statement}.
\end{proof}

%\vspace{-3ex}

We now describe the collaborative safety algorithm and how it may be used to communicate safety needs to neighboring agents in the formation control problem. See \cite{butler2023distributed} for a more detailed discussion on the construction of the collaborative safety algorithm with respect to a single safety condition for each agent. The central idea of this algorithm involves rounds of collaboration between agents, where each round of collaboration between agents, centered on an agent $i\in [n]$, involves the following steps:
\begin{enumerate}
    \item Receive (send) requests from (to) neighbors in $\mathcal{N}_i$
    \item Coordinate requests and determine needed compromises
    \item Send (receive) adjustments to (from) neighboring nodes in $\mathcal{N}_i$.
\end{enumerate}
This algorithm will return a set of constrained allowable filtered actions for each agent $\overline{\mathcal{U}}_i^s \subseteq \mathcal{U}_i^s$, where any safe action selected from this set will also be safe for all neighbors in $\mathcal{N}_i$.
In order to determine what requests should be made of neighbors, each agent must compute its \textit{maximum safety capability} with respect to the second-order safety condition as defined by \eqref{eq:phi3_separated}. However, since the safety capability of agent~$i$ with respect to multiple obstacles is represented as a vector, rather than a scalar value for a single condition \cite{butler2023distributed}, we must carefully define the maximum safety capability for agents in the context of formation control with multiple obstacles.

\subsection{Maximum Capability Given Multiple Obstacles} \label{sec:max_cap_comp}
To define the maximum capability of an agent~$i$ with respect to multiple obstacles, we begin by making the following assumption.
%\vspace{1ex}
\begin{assumption} \label{assume:nonempty_convex_contraints}
Let $\mathcal{U}_i^s$, defined in \eqref{eq:constraints_for_i}, be a non-empty convex set which is defined by $\mathcal{U}_i^s = \{u_i^s \in \mathbb{R}^{M_i}: G_iu_i^s - l_i \leq \mathbf{0}\}$,
\end{assumption}
\noindent
where $G_i \in \mathbb{R}^{Y_i \times M_i}, l_i \in \mathbb{R}^{Y_i}$, with $Y_i$ being the number of halfspaces whose intersection define the control constraint set for agent~$i \in [n]$, with $G_i$ and $l_i$ being defined by the given application and limitations of the controller for agent~$i$. In order to determine the ``safest" action agent~$i$ may take given multiple obstacles, we want to choose the action $u_i^s$ that maximizes the minimum entry of the vector $B_i u_i^s$ from \eqref{eq:multi-obs-compact}, which is defined by the following max-min optimization problem: 
\begin{equation}\label{eq:opt_capcity_input_action}
    \max_{u_i^s \in \mathcal{U}_i^s} \min_{1\leq k \leq |\mathcal{O}_i|} [B_i u_i^s]_{k}.
\end{equation}

%\vspace{-4ex}

This problem characterizes the optimal control strategy $u^*_i$ that attempts to satisfy the safety constraint \eqref{eq:phi3_separated}
% $[A_i(x_i)u_i]_{k^*}$ 
imposed on agent~$i$ for each obstacle $o \in \mathcal{O}_i$ that is at most risk of being violated (or being violated the worst).
We can rewrite \eqref{eq:opt_capcity_input_action} as a linear programming problem:
\begin{align}\label{eq:opt_capcity_input_action_LP}
    \min_{\xi_i}&~d^\top \xi_i \nonumber\\
    \text{s.t.}& \begin{bmatrix}
        \mathbf{0} & G_i\\
        \mathbf{1} & -B_i
    \end{bmatrix}\xi_i - \begin{bmatrix}
        l_i\\
        \mathbf{0}
    \end{bmatrix}\leq \mathbf{0},
\end{align}
where $d^\top = \begin{bmatrix}
    -1 & \mathbf{0}_{M_i}^\top
\end{bmatrix}$, $\xi_i^\top = \begin{bmatrix}
    \gamma_i & u_i^{s\top}
\end{bmatrix}$, and $\gamma_i \in \mathbb{R}$ is a scalar that captures the safety capability provided by the action~$u^*_i$.
The next proposition formally characterizes the equivalency of Problem~\eqref{eq:opt_capcity_input_action} and Problem~\eqref{eq:opt_capcity_input_action_LP}.
%\vspace{1ex}
\begin{proposition}\label{pp:equiv_opt_prob}
    Given Assumptions~\ref{assume:class_K_scalar}-\ref{assume:nonempty_convex_contraints}, the optimal solution of \eqref{eq:opt_capcity_input_action}:
    \begin{align*}
        u_i^* &= \arg\max_{u_i^s \in \mathcal{U}_i^s} \min_{1\leq k \leq |\mathcal{O}_i|} [B_i u_i^s]_{k}\\
        \gamma_i^* &= \max_{u_i^s \in \mathcal{U}_i^s} \min_{1\leq k\leq |\mathcal{O}_i|} [B_i u_i^s]_{k}
    \end{align*}
    exists if and only if there exists an optimal solution $\xi_i^{(*)} = \begin{bmatrix}
        \gamma_i^{(*)}& {u_i^{(*)}}^\top
    \end{bmatrix}^\top$,
    and 
    $\gamma_i^{(*)} = \gamma_i^*$, $u_i^{(*)} = u_i^*$, in \eqref{eq:opt_capcity_input_action_LP}.
\end{proposition}
%\vspace{-3ex}
\begin{proof}
    We first notice that the following two optimization problems are equivalent:
    \begin{align}
        &\begin{cases}\label{eq:pp:step1}
            \max_{u_i^s}\min_{1\leq k \leq |\mathcal{O}_i|}~&[B_i u_i^s]_{k}\\
            \text{s.t.}~& G_iu_i^s - l_i \leq \mathbf{0},\\
        \end{cases}\\
        &\begin{cases}\label{eq:pp:step2}
            \max_{u_i^s, \gamma_i}~&\gamma_i\\
            \text{s.t.}~&G_iu_i^s - l_i \leq \mathbf{0}\\
                       & \gamma_i \leq [B_i u_i^s]_{k}~\qquad 1\leq k \leq |\mathcal{O}_i|,
        \end{cases}
    \end{align}
    by substituting $\min_{1\leq k \leq |\mathcal{O}_i|}~[B_i(x_i)u_i^s]_{k}$ with an achievable lower bound $\gamma_i$ on each $[B_i(x_i)u_i^s]_{k}$, that is, $\gamma_i \leq [B_i(x_i)u_i^s]_{k}, 1\leq k \leq |\mathcal{O}_i|$.
    % Furthermore, $\gamma_i$ becomes one of the optimization variables.
    Furthermore, we can show that
    \eqref{eq:opt_capcity_input_action} is equivalent to \eqref{eq:pp:step1} by rewriting $u_i^s \in \mathcal{U}_i^s$ explicitly as an optimization constraint, and similarly, we can show that
    \eqref{eq:opt_capcity_input_action_LP} is equivalent to \eqref{eq:pp:step2} by setting $\xi_i = \begin{bmatrix}
        \gamma_i & u_i^{s\top}
    \end{bmatrix}^\top$, $d^\top = \begin{bmatrix}
        -1 & \mathbf{0}_{M_i}^\top
    \end{bmatrix}$, and realizing that $\argmax \gamma_i = \argmin -\gamma_i$.
    The transitivity of equivalence relations concludes the proof. 
\end{proof}

%\vspace{-3ex}

Thus, we have a method for computing a vector that represents the maximum capability of agent~$i$ with respect to multiple obstacles $\mathcal{O}_i$. If $\gamma_i$ is negative, then agent~$i$ will make a request to its neighboring agents that will limit their control actions $\overline{\mathcal{U}}_j^s$ to those that will satisfy $[\Phi_i]_k \geq 0, \forall k \in \mathcal{O}_i$, assuming agent~$i$ takes the action $u_i^*$. In the following section, we describe how our \textit{modified collaborative safety algorithm} incorporates this capability vector. 

\subsection{Collective Safety Through Collaboration} \label{sec:colab_safety_thm}
Given our addition to the collaborative safety algorithm from \cite{butler2023distributed} that incorporates multiple safety constraints, the computation steps of our algorithm remain unchanged in Algorithm~\ref{alg:colab_safety} since the communication of multiple safety constraints from one neighbor is equivalent to multiple neighbors communicating a single constraint in the computation of control restrictions. Thus, the Collaborate subroutine in Line 6 of Algorithm~\ref{alg:colab_safety} may be considered to be the same subroutine from \cite{butler2023distributed}.
Note that the major innovation for Algorithm~\ref{alg:colab_safety} with respect to \cite{butler2023distributed} is the inclusion of handling multiple requests for a single agent $i \in [n]$ as vectors $\bar{c}_i, \delta_i \in \mathbb{R}^{|\mathcal{O}_i|}$, which is accomplished by leveraging \eqref{eq:multi-obs-compact} and \eqref{eq:block_col_A_i}. Thus, we have the following result on the safety of collaborating agents under Algorithm~\ref{alg:colab_safety}.
%\vspace{1ex}

\begin{theorem} \label{thm:colab_safety_alg}
    Let Assumptions~\ref{assume:u_dot_func_u}-\ref{assume:nonempty_convex_contraints} hold for all $i \in [n]$. If Algorithm~\ref{alg:colab_safety} returns $\overline{\mathcal{U}}_i^s(x(t)) \neq \emptyset, \forall i \in [n], \forall t \in \mathcal{T}$, then \eqref{eq:problem_statement} yields $\bigcap_{o \in \mathcal{O}_i}\mathcal{C}_{i,o}^1 \cap \mathcal{C}_{i,o}^2$ forward invariant during $t \in \mathcal{T}$ for all $i \in [n]$.
\end{theorem}
%\vspace{-3ex}
\begin{proof}
If Algorithm~\ref{alg:colab_safety} returns $\overline{\mathcal{U}}_i^s(x(t)) \neq \emptyset, \forall i \in [n], \forall t \in \mathcal{T}$, then by Lemmas~\ref{lem:CCBF} and \ref{lem:equivalent_solution} we have that any action taken by any agent from these constrained control sets must render $\bigcap_{o \in \mathcal{O}_i}\mathcal{C}_{i,o}^1 \cap \mathcal{C}_{i,o}^2 \cap \mathcal{C}_{i,o}^3$ forward invariant for all $i \in [n]$. Thus, by applying control constraints $\overline{\mathcal{U}}_i^s(x(t))$ to \eqref{eq:problem_statement} for each agent, we have, by Lemma~\ref{lem:ACBF}, that $\bigcap_{o \in \mathcal{O}_i}\mathcal{C}_{i,o}^1 \cap \mathcal{C}_{i,o}^2$ is also forward invariant during $t \in \mathcal{T}$ for all $i \in [n]$.
\end{proof}

%\vspace{-2ex}

\begin{algorithm}
    \caption{Collaborative Safety}\label{alg:colab_safety}
    \begin{algorithmic}[1]
        \Initialize{
            $i \gets i_0$;
            $\overline{\mathcal{U}}_i^s \gets \mathcal{U}_i^s; \tau_i \gets 0; B_i \gets \eqref{eq:capability_vector_expanded};$ \\
            $\bar c_{ij}\gets \mathbf{0}$ and $\bar c_{ji}\gets \mathbf{0}~\forall j\in \mathcal{N}_i$   
        }
        \Repeat
            \State $\tau_i \gets \tau_i + 1$
            \State $u_i^* \gets \arg\max_{u_i^s \in \overline{\mathcal{U}}_i^s} \min_{k \in \left[|\mathcal{O}_i| \right]} [B_i u_i^s]_{k}$
            \State $\bar{c}_i \gets B_i u_i^* + q_i$ 
            \State $\delta_i,\, \overline{\mathcal{U}}_i^s, \{\bar{c}_{ij}, \bar{c}_{ji}\}_{j \in \mathcal{N}_i}$ 
            \Statex \hspace{10ex} $\gets$ Collaborate\((\bar{c}_i, \overline{\mathcal{U}}_i^s, \{\bar{c}_{ij}, \bar{c}_{ji}\}_{j \in \mathcal{N}_i})\)
        \Until{$\delta_i \geq \mathbf{0}$}
        \State\Return $\overline{\mathcal{U}}_i^s$
    \end{algorithmic}
\end{algorithm}

In \cite{butler2023distributed}, we provide conditions for algorithm convergence for all agents given assumptions on the relationship between neighbors' safety constraints $h_j(x), \forall j \in \mathcal{N}_j$ for each node $i \in [n]$, which will be reintroduced in Section~\ref{sec:conflicting_safety}; however, we do not determine the rate at which a system will converge to a safe action for all agents if it exists, which is of critical importance in real-time applications such as formation control and obstacle avoidance. Therefore, in the following section, we build upon the analysis presented in \cite{butler2023distributed} to investigate the convergence rate for Algorithm~\ref{alg:colab_safety} 
to return either a feasible set of safe actions for all agents or verify that it is in a terminally infeasible state. To assist in the discussion on the convergence rate for Algorithm~\ref{alg:colab_safety}, we include the counter variable $\tau_i$ in the repeat loop, which is used to define collaborative rounds.
%\vspace{1ex}
\begin{definition} \label{def:round_com}
    In the Algorithm~\ref{alg:colab_safety}, one \textbf{round of collaboration}, counted by $\tau_i$, counts one completion of the Collaborate$(\bar{c}_i, \overline{\mathcal{U}}_i^s, \bar{c}_{ij}; \forall j \in \mathcal{N}_i)$ subroutine in Line 6 of Algorithm~\ref{alg:colaborate}.
\end{definition}

With the connection of Algorithm~\ref{alg:colab_safety} to forward invariance and safety established, we are prepared to discuss how our algorithm resolves the complexities of potentially conflicting safety needs between neighbors and how quickly such requests can be resolved under key assumptions described in the following sections.

\section{Resolving Conflicting Safety Needs}\label{sec:conflicting_safety}
In this section, we discuss key concepts and definitions for resolving scenarios in the execution of Algorithm~\ref{alg:colab_safety} where neighbors in $\mathcal{N}_i$ may make infeasible, or even conflicting, requests of a given agent $i \in [n]$. 
For completeness, we include the full descriptions of both subroutines for neighbor collaboration in Algorithm~\ref{alg:colaborate} and for coordinating safety requests in Algorithm~\ref{alg:coordinate}. 
By way of analogy, Algorithms~\ref{alg:colab_safety}-\ref{alg:coordinate} may be viewed as a type of distributed resource (safety) allocation process, where each agent has a finite control budget (represented by control constraints $\mathcal{U}_i^s$) with individual safety requirements (defined by $h_{i,o}(x_i,x_o), \forall o \in \mathcal{O}_i$) that inform the collaborative safety condition in \eqref{eq:lem_compact_condition}. Thus, in Algorithms~\ref{alg:colab_safety}-\ref{alg:coordinate}, each agent attempts to guarantee that \eqref{eq:lem_compact_condition} is satisfied by allotting safety deficits $\delta_{ij}$ to neighbors $j \in \mathcal{N}_i$, where neighbors can, in turn, communicate safety shortage adjustments $\varepsilon_{ij}$ if the requested allotment is not possible.

Thus, in Section~\ref{sec:interferring_constr}, we introduce the notion of weakly non-interfering constraints between neighbors in $\mathcal{N}_i$ for agent~$i$, which is used in \cite{butler2023distributed} to guarantee the asymptotic convergence of Algorithm~\ref{alg:colab_safety}.
A necessary addition for the finite time convergence of Algorithm~\ref{alg:colab_safety} is the definition of the protocol GetClosestPoint in Algorithm~\ref{alg:getClosestPoint} in Section~\ref{sec:get_closest_point}, a subroutine of Algorithm~\ref{alg:coordinate}, which is used to select control actions that balance the needs of jointly infeasible safety constraints between neighbors. Note that although the language used in this section refers to safety constraints relating to multiple obstacles, the principles can be readily generalized to any application with multiple safety constraints for dynamically coupled agents.

\begin{algorithm}
\caption{Collaborate}\label{alg:colaborate}
    \begin{algorithmic}[1]
        \Initialize{
            $i \gets i_0; w\gets w_0; \overline{\mathcal{N}}_i\gets \emptyset$
        }
        \Inputs{
            \(\bar{c}_i; \overline{\mathcal{U}}_i^s; \{\bar{c}_{ij}, \bar{c}_{ji}\}_{j \in \mathcal{N}_i}\)
        }
        \Repeat
            \State $\upsilon_i \gets \upsilon_i + 1$
            \State $\delta_i \gets \bar{c}_i - \sum_{j \in \mathcal{N}_i}\bar{c}_{ij}$
            \State $\{\delta_{ij}\}_{j\in \mathcal{N}_i} \gets \left\{ \frac{\delta_i w_{ij}}{\sum_{l \in \mathcal{N}_i \setminus \overline{\mathcal{N}}_i} w_{il}}\right\}_{j \in \mathcal{N}_i}$
            \State SEND to each $j \in \mathcal{N}_i \setminus \overline{\mathcal{N}}_i : \delta_{ij}$
            \State RECEIVE from all $j \in \mathcal{N}_i : \{\delta_{ji}\}_{j\in\mathcal{N}_i}$
            \State {$\overline{\mathcal{U}}_i^s, \{\bar{c}_{ji}, \varepsilon_{ji}\}_{j \in \mathcal{N}_i} \gets$ Coordinate$(\{\bar{c}_{ji}, \delta_{ji}\}_{j \in \mathcal{N}_i})$}
            \State SEND to each $j \in \mathcal{N}_i : \varepsilon_{ji}$
            \State RECEIVE from all $j \in \mathcal{N}_i \setminus \overline{\mathcal{N}}_i : \{\varepsilon_{ij}\}_{j\in\mathcal{N}_i}$

            \For{$j \in \mathcal{N}_i$}
                \State $\bar{c}_{ij} \gets \bar{c}_{ij} + \delta_{ij} + \varepsilon_{ij}$ 
                \If{$\exists k \in \left[|\mathcal{O}_i| \right] \text{ s.t. } [\varepsilon_{ij}]_k > 0$}
                    \State $\overline{\mathcal{N}}_i \gets \overline{\mathcal{N}}_i \cup \{ j\}$
                \EndIf
            \EndFor
        \Until{{$(\overline{\mathcal{N}}_i = \mathcal{N}_i) \lor (\varepsilon_{ij} = \mathbf{0}  \land \varepsilon_{ji} = \mathbf{0}; \forall j \in \mathcal{N}_i)$}}
        \State\Return $\delta_i,\, \overline{\mathcal{U}}_i^s, \{\bar{c}_{ij}, \bar{c}_{ji} \}_{j \in \mathcal{N}_i}$
    \end{algorithmic}
\end{algorithm}

\begin{algorithm}
\caption{Coordinate}\label{alg:coordinate}
    \begin{algorithmic}[1]
        \Initialize{
            $i \gets i_0; \varepsilon_{ji} \gets \mathbf{0}$ and $A_{ji,\mathcal{O}_j} \gets \eqref{eq:block_col_A_i}, \forall j \in \mathcal{N}_i$
        }
        \Inputs{
            $\{\bar{c}_{ji}, \delta_{ji}\}_{j \in \mathcal{N}_i}$
        }
        \For{$j \in \mathcal{N}_i$}
           \State $\overline{\mathcal{U}}_{ji} \gets \{ u_i \in \mathbb{R}^{M_i}: A_{ji,\mathcal{O}_j} u_i + \bar{c}_{ji} + \delta_{ji} \geq \mathbf{0} \}$
        \EndFor

        \State $\overline{\mathcal{U}}_{\mathcal{N}_i} \gets \bigcap_{j \in \mathcal{N}_i} \overline{\mathcal{U}}_{ji}$

        \If{$\mathcal{U}_i \cap \overline{\mathcal{U}}_{\mathcal{N}_i} \neq \emptyset$ }
            \State $\overline{\mathcal{U}}_i^s \gets \mathcal{U}_i \cap \overline{\mathcal{U}}_{\mathcal{N}_i}$

        \Else
            \State $\overline{u}_i \gets$ {\small GetClosestPoint$(\mathcal{U}_i,\overline{u}_i^{\upsilon}, \{\overline{\mathcal{U}}_{ji}, \delta_{ji}, \delta_{ji}^{\upsilon}\}_{j \in \mathcal{N}_i})$}
            % \EndIf
            \State $\overline{\mathcal{U}}_i^s \gets \{ \overline{u}_i \}$  
            \State $\overline{u}_i^{\upsilon} \gets \overline{u}_i$
            
            \For{$j \in \mathcal{N}_i$}
                \If{$\exists k \in \left[|\mathcal{O}_j| \right] \text{ s.t. } [A_{ji}\overline{u}_i + \bar{c}_{ji} + \delta_{ji}]_k < 0$}
                    \State $\varepsilon_{ji} \gets -(A_{ji}\overline{u}_i + \bar{c}_{ji} + \delta_{ji})$
                \EndIf
            \EndFor
        \EndIf
        \State $\{\delta_{ji}^{\upsilon}\}_{j \in \mathcal{N}_i} \gets \{\delta_{ji}\}_{j \in \mathcal{N}_i}$
        \State $\{\bar{c}_{ji}\}_{j \in \mathcal{N}_i} \gets \{\bar{c}_{ji} + \delta_{ji} + \varepsilon_{ji}\}_{j \in \mathcal{N}_i}$
        \State \Return $\overline{\mathcal{U}}_i^s,\{\bar{c}_{ji}, \varepsilon_{ji}\}_{j \in \mathcal{N}_i}$
    \end{algorithmic}
\end{algorithm}

\subsection{Weakly Non-Interfering Safety Constraints} \label{sec:interferring_constr}

In this section, we discuss the relationship of neighbors' constraints to each other for a given node $i \in [n]$ and define properties of constraint interference.
Similarly to the analysis performed in \cite{butler2023distributed}, we narrow the focus of our discussion to when neighboring requests are sufficiently non-interfering, defined as follows. 
%\vspace{1ex}
\begin{definition} \label{def:weakly_non_interfering}
     For a given agent $i \in [n]$, the set of neighbor constraints $h_{j}(x_j, x_{o})$ for $j \in \mathcal{N}_i$ and $o \in \mathcal{O}_j$ are said to be \textbf{weakly non-interfering} if there exists a vector $a \in \mathbb{R}^{M_i}$ such that $a \cdot a_{ji,o}(\mathbf{x}_j, x_{o}) > 0, \forall j \in \mathcal{N}_i, \forall o \in \mathcal{O}_j$, where $a_{ji,o}(\mathbf{x}_j, x_{o})$ is defined by \eqref{eq:neighbor_effects}.
\end{definition}

Using the definition of weakly non-interfering constraints, we obtain the following result regarding the feasibility of safe control actions in the unrestricted control space $\mathbb{R}^{M_i}$ for neighbors $j \in \mathcal{N}_i$ through the compositions of neighbors' requests 
\begin{equation} \label{eq:safe_actions_for_neighbors}
    \overline{\mathcal{U}}_{ji} = \{ u_i \in \mathbb{R}^{M_i}: A_{ji,\mathcal{O}_j} u_i + \bar{c}_{ji} + \delta_{ji} \geq \mathbf{0} \}.
\end{equation}

\begin{proposition}
    If the set of constraints $h_{j}(x_j, x_{o})$ for $j \in \mathcal{N}_i$ and $o \in \mathcal{O}_j$ are weakly non-interfering, then $\bigcap_{j \in \mathcal{N}_i} \overline{\mathcal{U}}_{ji} \neq \emptyset$.
\end{proposition}

This proof of this result follows directly from \cite[Lemma 2]{butler2023distributed}. We find a relationship between the dimension of the control space $M_i$ for a given agent $i\in [n]$ and the evaluation of constraint sets being weakly non-interfering based on the total number of requests received by agent~$i$. 

%\vspace{1ex}
\begin{proposition}
    Let $Z_i \in \mathbb{Z}_{\geq 0}$ be the total number of requests received by agent $i \in [n]$, i.e., $Z_i \leq \sum_{j \in \mathcal{N}_i} |\mathcal{O}_j|$. If $Z_i \leq M_i$, then any set of neighbor constraints $h_{j}(x_j, x_{o})$ for $j \in \mathcal{N}_i$ and $o \in \mathcal{O}_j$ are weakly non-interfering.
\end{proposition}
%\vspace{-3ex}
\begin{proof}
For any set of $Z_i$ vectors $v_1, \dots, v_{Z_i} \in \mathbb{R}^{M_i}$, if $Z_i \leq M_i$ then there will always exist a halfspace in $\mathbb{R}^{M_i}$ defined by the vector $a \in \mathbb{R}^{M_i}$ 
\begin{equation*}
    \mathcal{U}^a = \{ u \in \mathbb{R}^{M_i}: a \cdot u > 0 \}, 
\end{equation*}
such that $v_1, \dots, v_{Z_i} \in \mathcal{U}^a$. Thus, every set of $Z_i$ vectors where $Z_i \leq M_i$ will be weakly non-interfering.
\end{proof}

%\vspace{-2ex}

Although requiring neighbors' requests to be weakly non-interfering guarantees that there will exist some jointly feasible control space in $\mathbb{R}^{M_i}$, incorporating control constraints $\mathcal{U}_i^s \subset \mathbb{R}^{M_i}$ may eliminate any jointly feasible control inputs. We define neighbors' constraints that cause such infeasibility as follows.
%\vspace{1ex}
\begin{definition}
    For a given agent $i \in [n]$, the neighbor constraint $h_{j}(x_j, x_{o}), j \in \mathcal{N}_i, o \in \mathcal{O}_j$ is \textbf{infeasible} if $\mathcal{U}_i \cap \overline{\mathcal{U}}_{ji} = \emptyset$. Further, the neighbor constraint pair $h_{j}(x_j, x_{o})$ and $h_{l}(x_l, x_{p})$ for $j,l \in \mathcal{N}_i, o \in \mathcal{O}_j, p \in \mathcal{O}_l$ are \textbf{jointly infeasible} if $\mathcal{U}_i \cap \overline{\mathcal{U}}_{ji} \cap \overline{\mathcal{U}}_{li} = \emptyset$.
\end{definition}

\noindent 
In Figure~\ref{fig:neighbor_req_diagram}, we illustrate an example of two jointly infeasible constraints for an agent $i\in [n]$ with $u_i \in \mathbb{R}^2$. Note that any infeasible constraint is simultaneously jointly infeasible with every other neighbor's constraint.

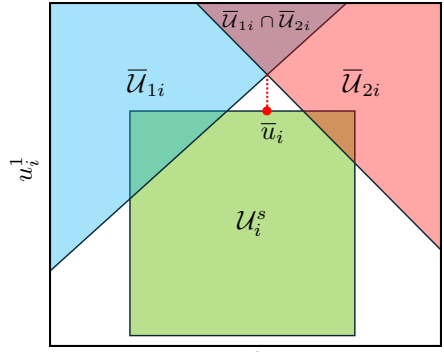
\begin{figure}
    \centering
    \begin{overpic}[width=.6\columnwidth]{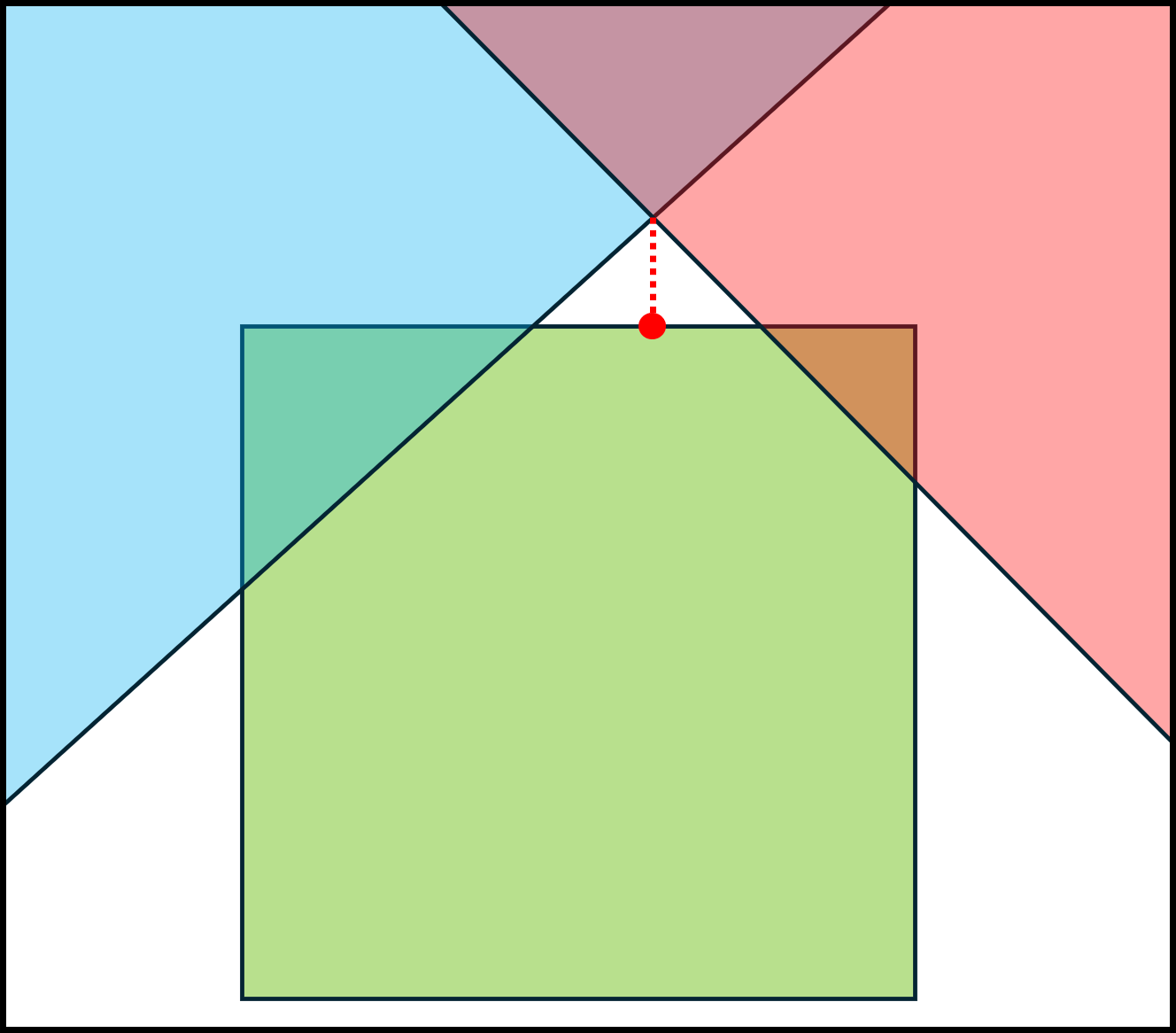}
    \put(-10,42){\parbox{0.75\linewidth}\normalsize \rotatebox{90}{$u_i^1$}}
    \put(48,-7){\parbox{0.75\linewidth}\normalsize{$u_i^2$}}
    \put(48,30){\parbox{0.75\linewidth}\normalsize{$\mathcal{U}_i^s$}}
    \put(54,53){\parbox{0.75\linewidth}\normalsize{$\overline{u}_i$}}
    \put(75,65){\parbox{0.75\linewidth}\normalsize{$\overline{\mathcal{U}}_{2i}$}}
    \put(20,65){\parbox{0.75\linewidth}\normalsize{$\overline{\mathcal{U}}_{1i}$}}
    \put(44.3,81.5){\parbox{0.75\linewidth}\normalsize{\scriptsize $\overline{\mathcal{U}}_{1i} \cap \overline{\mathcal{U}}_{2i}$}}
    \end{overpic}
    %\vspace{2ex}
    \caption{
        An example of when agent $i \in [n]$ is constrained by two neighbors, where $u_i \in \mathbb{R}^2$. The constrained control space for agent~$i$, $\mathcal{U}_i^s$, is shaded green, with the feasibly safe control actions for neighbors 1 and 2 shaded in blue and red, respectively. Both $\overline{\mathcal{U}}_{1i}$ and $\overline{\mathcal{U}}_{2i}$ are individually feasible, but jointly infeasible, with the set of feasibly safe control actions for neighbors 1 and 2 shown in the purple-shaded region. The compromise-seeking action $\overline{u}_i \in \mathbb{R}^2$ chosen by Algorithm~\ref{alg:getClosestPoint} is marked on the boundary of $\mathcal{U}_i^s$, which is the closest action in $\mathcal{U}_i^s$ to the feasibly safe control actions for both neighbors $\overline{\mathcal{U}}_{1i} \cap \overline{\mathcal{U}}_{2i}$.
    }
    \label{fig:neighbor_req_diagram}
\end{figure}
%\vspace{1ex}
\begin{lemma}\label{lem:infeas_implies_joint_infeas}
    If $h_{j}(x_j, x_{o}), j \in \mathcal{N}_i, o \in \mathcal{O}_j,$ is infeasible, then $h_{j}(x_j, x_{o})$ is jointly infeasible with $h_{l}(x_l, x_{p}), \forall l \in \mathcal{N}_i, p \in \mathcal{O}_l$.
\end{lemma}
%\vspace{-3ex}
\begin{proof}
    If $\mathcal{U}_i \cap \overline{\mathcal{U}}_{ji} = \emptyset$, then its intersection with any constraint set $\overline{\mathcal{U}}_{li}$ will also be empty, i.e. $\mathcal{U}_i \cap \overline{\mathcal{U}}_{ji} \cap \overline{\mathcal{U}}_{li} = \emptyset, \forall l \in \mathcal{N}_i, p \in \mathcal{O}_l$. Thus, $h_{j}(x_j, x_{o})$ is jointly infeasible with $h_{l}(x_l, x_{p}), \forall l \in \mathcal{N}_i, p \in \mathcal{O}_l$.
\end{proof}

%\vspace{-2ex}

Jointly infeasible constraints create the potential for no allowable control action from each agent in the system to satisfy all safety constraints.
%\vspace{1ex}
\begin{definition} \label{def:terminally_infeasible}
    We say that $(\mathbf{x}_i,\mathbf{x}_{\mathcal{O}_i}), \forall i \in [n]$ is a \textbf{terminally infeasible state} for a network of cooperating agents if there does not exist a set of control inputs $u_1^s, \dots, u_n^s \in \mathcal{U}_1^s \times \cdots \times \mathcal{U}_n^s$ such that 
    \begin{equation*}
        \Phi_i(\mathbf{x}_i,\mathbf{x}_{\mathcal{O}_i}, u_i^s, u_{\mathcal{N}_i}^s) \geq 0, \forall i \in [n].
    \end{equation*}
\end{definition}

\noindent
Thus, for a given state $(\mathbf{x}_i,\mathbf{x}_{\mathcal{O}_i}), \forall i \in [n]$, Algorithm~\ref{alg:colab_safety} will either terminate with at least one safe action for all agents, or $(\mathbf{x}_i,\mathbf{x}_{\mathcal{O}_i}), \forall i \in [n]$ will be a terminally infeasible state.
Since real-world applications in formation control will require reliably fast decision-making for all agents, we wish to investigate the rate at which Algorithm~\ref{alg:colab_safety} will either terminate with a safe action or identify in finite-time a terminally infeasible state according to agent safety requirements. 

\subsection{Compromise-Seeking and Maximally Beneficial Action} \label{sec:get_closest_point}

We define an additional term for counting the number of iterations carried out by Algorithm~\ref{alg:colaborate}, where we again define a counting variable $\upsilon_i$ to assist in our convergence discussions.
%\vspace{1ex}
\begin{definition} \label{def:negotiation_round}
    One \textbf{round of negotiation} is counted as a single iteration of the repeat loop of Algorithm~\ref{alg:colaborate}.
\end{definition}

\noindent
To encourage linear-time convergence of Algorithm~\ref{alg:colab_safety}, we define a function for choosing the closest action between a jointly infeasible set of constraints and the control constraints $\mathcal{U}_i$ of a given agent~$i$. We begin by defining the problem of finding the closest point between two disjoint, nonempty, polytopic convex hulls in $\mathbb{R}^{M_i}$ as
%\vspace{-2ex}
\begin{equation} \label{eq:closest_point}
    \begin{aligned}
        \min_{z_1,z_2} \quad & {\frac{1}{2}\left\Vert z_1 - z_2 \right\Vert}_2 \\
        \text{s.t.} \quad & G_1 z_1 - l_1 \leq 0 \\
        & G_2 z_2 - l_2 \leq 0,
    \end{aligned}
\end{equation}

%\vspace{-2.5ex}

where $G_1 \in \mathbb{R}^{Z_1 \times M_i}, l_1 \in \mathbb{R}^{Z_1}$ and $G_2 \in \mathbb{R}^{Z_2 \times M_i}, l_2 \in \mathbb{R}^{Z_2}$ are the matrix-vector pairs that encode the collection of $Z_1, Z_2 > 1$ halfspaces whose intersection defines the first and second convex hulls, respectively. We can restructure \eqref{eq:closest_point} in the form of a quadratic program as

%\vspace{-2ex}

\begin{equation} \label{eq:closest_point_qp}
    \begin{aligned}
        \min_{\xi} \quad &
        \frac{1}{2}
        \left\Vert
        \begin{bmatrix}
        I & -I    
        \end{bmatrix} \xi \right\Vert_2 \\
        \text{s.t.} \quad &
        \begin{bmatrix}
            G_1 & \mathbf{0} \\
            \mathbf{0} & G_2
        \end{bmatrix}
        \xi -
        \begin{bmatrix}
            b_1 \\ b_2
        \end{bmatrix}
        \leq 0,
    \end{aligned}
\end{equation}

\noindent
where $\xi = [z_1, z_2]^\top$. Note that
\begin{equation*}
    \left\Vert 
    \begin{bmatrix}
        I & -I    
    \end{bmatrix} \xi \right\Vert_2
    = 
    \xi^\top
    \begin{bmatrix}
        I & -I \\
        -I & I
    \end{bmatrix}
    \xi,
\end{equation*}
which yields the quadratic term. 

Given this formulation, we define an algorithm for selecting the closest point between two convex hulls that also considers requests from previous rounds of negotiation.

\begin{algorithm}
\caption{GetClosestPoint}\label{alg:getClosestPoint}
    \begin{algorithmic}[1]
        \Initialize{
            $i \gets i_0$
        }
        \Inputs{
        $\mathcal{U}_i,\overline{u}_i^{\upsilon}, \{\overline{\mathcal{U}}_{ji}, \delta_{ji}, \delta_{ji}^{\upsilon}\}_{j \in \mathcal{N}_i}$
        }
        \State $\mathcal{N}_i^{\upsilon} \gets \left\{j \in \mathcal{N}_i: \exists k \in \left[|\mathcal{O}_j| \right] \text{ s.t. } [\delta_{ji}^{\upsilon}]_k < 0 \right\}$
        \State $\mathcal{N}_i^{\upsilon+1} \gets \left\{j \in \mathcal{N}_i: \exists k \in \left[|\mathcal{O}_j| \right] \text{ s.t. }[\delta_{ji}]_k < 0 \right\}$
        \If{$\mathcal{N}_i^{\upsilon} = \emptyset$}
            \State $\overline{u}_i \gets$ Solve \eqref{eq:closest_point_qp} for $\left(\mathcal{U}_i,\bigcap_{j \in \mathcal{N}_i} \overline{\mathcal{U}}_{ji} \right)$ 
        \Else
            \If{$\mathcal{U}_i \cap \left(\bigcap_{j \in \mathcal{N}_i^{\upsilon+1}} \overline{\mathcal{U}}_{ji} \right) = \emptyset$}
                \State $\overline{u}_i \gets$ Solve \eqref{eq:closest_point_qp} for $\left(\mathcal{U}_i,\bigcap_{j \in \mathcal{N}_i^{\upsilon+1}} \overline{\mathcal{U}}_{ji} \right)$
            \Else
                \State $\overline{U}_i \gets \partial \mathcal{U}_i \cap \left(\bigcap_{j \in \mathcal{N}_i^{\upsilon+1}} \overline{\mathcal{U}}_{ji} \right)$ 
                \State $\overline{u}_i \gets \min_{u_i \in \overline{U}_i} \Vert u_i - \overline{u}_i^\upsilon \Vert_2$
            \EndIf
        \EndIf
        \State \Return $\overline{u}_i$
    \end{algorithmic}
\end{algorithm}

\noindent
We can describe the procedure of Algorithm~\ref{alg:getClosestPoint} in words as follows:
\begin{enumerate}
    \item If there are no requests from the previous negotiation round, then compute the closest point using \eqref{eq:closest_point_qp}.
    \item Otherwise, compute the closest point to the current set of requests from this negotiation round using \eqref{eq:closest_point_qp}.
    \item If the current set of requests are feasible, then project the previous action $\overline{u}_i^\upsilon$ to the curve that is found by taking the intersection of the current requests with $\partial \mathcal{U}_i$.
    % (which can happen when the number of current requests is less than $M_i$) 
\end{enumerate}

\noindent
This protocol leverages a pseudo-greedy approach to computing a compromise-seeking action for infeasible neighbors by complying fully with the most recent requests made by neighbors. Thus, rather than converging asymptotically (as was the case for \cite[Algorithm 1]{butler2023distributed}), Algorithm~\ref{alg:getClosestPoint} enables finite-time convergence for Algorithm~\ref{alg:colab_safety}, which is demonstrated in Section~\ref{sec:lin_alg_convergence}.

Given that Algorithm~\ref{alg:getClosestPoint} is responsible for selecting an action $\overline{u}_i \in \partial \mathcal{U}_i$ that is a suitable compromise between neighbors with jointly infeasible requests for any given round of collaboration, we define the term compromise-seeking action as follows.
%\vspace{1ex}
\begin{definition}
    For a given agent $i \in [n]$, a \textbf{compromise-seeking action} $\overline{u}_i \in \mathbb{R}^{M_i}$ for a subset of jointly infeasible neighbor constraints in $\mathcal{I} \subseteq \mathcal{N}_i$ is the closest point on $\partial \mathcal{U}_i$ to the non-empty, disjoint, convex hull defined by $\bigcap_{j \in \mathcal{I}}\overline{\mathcal{U}}_{ji}.$  
\end{definition}
\noindent
See Figure~\ref{fig:neighbor_req_diagram} for an example of a compromise-seeking action that would be selected by Algorithm~\ref{alg:getClosestPoint} for two jointly infeasible constraints. Note that by Algorithms~\ref{alg:coordinate} and \ref{alg:getClosestPoint}, any potential adjustment $\varepsilon_{ji}$, computed and sent by agent~$i$, is a function of the compromise-seeking action  $\overline{u}_i \in \partial \mathcal{U}_i$ as
\begin{equation}\label{eq:eps_funct_of_comp_pt}
    \varepsilon_{ji}(\overline{u}_i) = -(A_{ji,\mathcal{O}_j}\overline{u}_i + \bar{c}_{ji} + \delta_{ji})
\end{equation}
where an adjustment is sent only if $\exists k \in [|\mathcal{O}_j|]$ such that $[\varepsilon_{ji}(\overline{u}_i)]_{k} > 0$.
We define a few additional terms related to the relative benefit of a given control input $\overline{u}_i$ towards assisting in the safety of each neighbor $j \in \mathcal{N}_i$.
%\vspace{1ex}
\begin{definition}
    For a given agent $i\in [n]$, the control action $\overline{u}_i \in \partial \mathcal{U}_i$ is \textbf{maximally beneficial} for neighbor $j \in \mathcal{N}_i$ with respect to obstacle $o \in \mathcal{O}_j$ if
    \begin{equation}\label{eq:max_benifit_neighbor}
        \overline{u}_i = \argmax_{u_i \in \mathcal{U}_i} a_{ji,o}(\mathbf{x}_j, x_{o})u_i.
    \end{equation}
\end{definition}

%\vspace{-3ex}

When an agent~$i$ selects an action that is maximally beneficial to 
help neighbor $j$ to avoid an obstacle $o$,
% \humph{avoiding} an obstacle of neighbor $j$, 
any other action taken by $i$ will be equally or less helpful, i.e. $a_{ji,o}(\mathbf{x}_j, x_{o}) \overline{u}_i \geq a_{ji,o}(\mathbf{x}_j, x_{o})u_i, \forall u_i \in \mathcal{U}_i$. This may cause agent~$i$ to become fully constrained with respect to the requests of neighbor~$j$, defined as follows.
%\vspace{1ex}
\begin{definition} \label{def:fully_constrained}
    Agent $i \in [n]$ is \textbf{fully constrained} with respect to neighbor $j \in  \mathcal{N}_i$, at negotiation round $\upsilon \geq 1$, if $\delta_{ji} + \varepsilon_{ji}(\overline{u}_i) =~\mathbf{0}$.
\end{definition}
\noindent
To assist in the proof of finite-time algorithm convergence, we provide the following lemma on maximally beneficial actions of an agent~$i$ for a single request from neighbor $j$.
%\vspace{1ex}
\begin{lemma} \label{lem:max_ben_iff_full_cons}
    Let agent $i \in [n]$ receive only one request $\delta_{ji} \leq 0$ from neighbor $j \in \mathcal{N}_i$ with $|\mathcal{O}_j| = 1$. If $\varepsilon_{ji}(\overline{u}_i) > 0$, as computed in Line 15 of Algorithm~\ref{alg:coordinate}, then $\overline{u}_i \in \partial \mathcal{U}_i$ is maximally beneficial for neighbor $j \in \mathcal{N}_i$ and agent~$i$ will be fully constrained with respect to agent~$j$ for all subsequent rounds of negotiation. 
\end{lemma}
%\vspace{-3ex}
\begin{proof}
    If $\varepsilon_{ji}(\overline{u}_i) > 0$, then by Algorithms~\ref{alg:coordinate} and \ref{alg:getClosestPoint}, node $i$ will have selected $\overline{u}_i$ by solving \eqref{eq:closest_point_qp} for $\left(\mathcal{U}_i, \overline{\mathcal{U}}_{ji} \right)$, where $\overline{\mathcal{U}}_{ji}$ is defined by \eqref{eq:safe_actions_for_neighbors}. Since $|\mathcal{O}_j| = 1$, solving \eqref{eq:closest_point_qp} for $\overline{u}_i \in \partial \mathcal{U}_i$ simplifies to \eqref{eq:max_benifit_neighbor}. Thus, the compromise-seeking action $\overline{u}_i \in \partial \mathcal{U}_i$ is maximally beneficial for neighbor $j \in \mathcal{N}_i$. Further, since agent~$i$ has already selected the action $\overline{u}_i$ that is most helpful for agent~$j$ (i.e., since $a_{jk,o}(\mathbf{x}_j, x_{o}) \overline{u}_k \geq a_{jk,o}(\mathbf{x}_j, x_{o})u_k^s, \forall u_k^s \in \mathcal{U}_k^s$), then, by Line 20 of Algorithm~\ref{alg:coordinate}, $\bar{c}_{ji}$ must remain constant at every negotiation round, implying that $\delta_{ji} = \varepsilon_{ji}(\overline{u}_i)$, making agent~$i$ fully constrained with respect to agent~$j$ by Definition~\ref{def:fully_constrained}.
\end{proof}

\section{Linear-Time Algorithm Convergence}\label{sec:lin_alg_convergence}
In this section, we investigate the sufficient conditions under which Algorithm~\ref{alg:colab_safety} will converge in linear time to viable safe actions for all neighbors if they exist. We describe three conditions that are sufficient for the linear time complexity of Algorithm~\ref{alg:colab_safety}. These conditions are, from the most general to the most specific in our application: Synchronization, a tree-like graph structure, and limited 1-hop initial interaction, which are described by Assumptions~\ref{assume:syncronized}-\ref{assume:only_1hop_requests}, respectively.

\subsection{Communication Synchronization} \label{sec:com_sync}
Note that by defining the passage of time through rounds of communication in Definitions~\ref{def:round_com} and \ref{def:negotiation_round} there is an implied property of synchronization that must hold for all cooperating agents. 
%\vspace{1ex}
\begin{assumption} \label{assume:syncronized}
     For any given agent $i \in [n]$, let $\tau_i = \tau_j$ and $\upsilon_i = \upsilon_j, \forall j \in \mathcal{N}_i$.
\end{assumption}

\noindent
In other words, we require that, before any agent proceeds with a new round of collaboration (i.e. before entering Line 6 of  Algorithm~\ref{alg:colab_safety}), all other neighbors must have also completed a round of collaboration. This assumption also recursively implies synchronization across the entire coupled cooperating system. To simplify notation, under Assumption~\ref{assume:syncronized}, we denote the collaborative round as $\tau$ and negotiation round as $\upsilon$ for the entire system.

\subsection{Tree-Structured Communication Graph} \label{sec:case_study}
As an initial study into how interfering safety constraints resolve in Algorithm~\ref{alg:colab_safety}, we begin by considering a simplified version of the communication structure within a multi-agent system. We can describe the communication graph of a network with a binary adjacency matrix $G(\mathcal{E}) \in \{0,1\}^{n \times n}$, where a communication link from agent~$j$ to agent~$i$ is represented with the entry $G_{ij} = 1$ such that if $(i,j) \in \mathcal{E}$, then $G_{ij} = G_{ji} = 1$, otherwise $G_{ij} = G_{ji} = 0$. 
In other words, we assume the graph describing the dynamic coupling between agents is equivalent to the communication graph and that the communication graph is undirected. 

To simplify this initial analysis of the consequences of infeasible constraints, we impose the following assumption of the flow of communication in the graph. 
%\vspace{1ex}
\begin{assumption} \label{assume:tree_structure}
    The communication graph $G$ is an undirected graph in which any two vertices $i,j \in [n]$ are connected by exactly one path.
\end{assumption}
\noindent
In other words, Assumption~\ref{assume:tree_structure} requires that $G$ has a tree structure. While this assumption is fairly restrictive on the structure of the communication graph, it allows us to consider the propagation of infeasible requests through the network with respect to a single root node without requiring consideration of cascading infeasible requests that may propagate or cycle infinitely in a generic network as a result of a single infeasible request. In the remainder of our discussion, we use the following terminology leveraging the structure imposed by Assumption~\ref{assume:tree_structure}. A \textit{root node} of interest denotes a node $i\in [n]$ which may be receiving requests from \textit{1-hop neighbors}, which includes all nodes $j \in \mathcal{N}_i$, i.e., neighbors with a direct communication link to the root node. Note that by Assumption~\ref{assume:tree_structure} there cannot exist any communication links between 1-hop neighbors. Further, we refer to \textit{2-hop neighbors} with respect to a given root node as simply the 1-hop neighbors of the root node's 1-hop neighbors, which includes all nodes $k \in \mathcal{N}_j \setminus \{i\}$ for nodes $j \in \mathcal{N}_i$. 

\subsection{Limited 1-Hop Neighborhood Interaction} \label{sec:only-1-hop-neighborhood-req}
In this section, we define a condition that limits the initial interaction of collaborating neighbors such that the convergence rate analysis of Algorithm~\ref{alg:colab_safety} is simplified.
%\vspace{1ex}
\begin{assumption} \label{assume:only_1hop_requests}
    At $\tau = \upsilon = 1$, there exists only one node $i \in  [n]$ such that $\delta_{ji} \leq \mathbf{0}~\forall j \in \mathcal{N}_{i}$, and $\delta_{kl} = \mathbf{0}~\forall k \in [n] \setminus (\mathcal{N}_i \cup \{ i \})\text{ and } l \in \mathcal{N}_k$. Further, let $|\mathcal{O}_j| = 1~\forall j \in \mathcal{N}_i \cup \{ i \}$.
\end{assumption}

\noindent
In other words, at the first round of collaboration for a group of collaborating nodes using Algorithm~\ref{alg:colab_safety}, safety requests are sent only by nodes in the 1-hop neighborhood of node $i$ with exactly one request each (mandated by $| \mathcal{O}_i| = 1)$. 
This assumption simplifies analyzing the convergence of Algorithm~\ref{alg:colab_safety} by allowing only one set of (potentially) interfering requests at a single node $i\in [n]$. One may consider Assumption~\ref{assume:only_1hop_requests} to be a specific initial condition for the system at the start of collaboration using Algorithm~\ref{alg:colab_safety}, as illustrated by Figure~\ref{fig:tree_case_diagram}.

\begin{figure}
    \centering
    \begin{overpic}[width=.8\columnwidth]{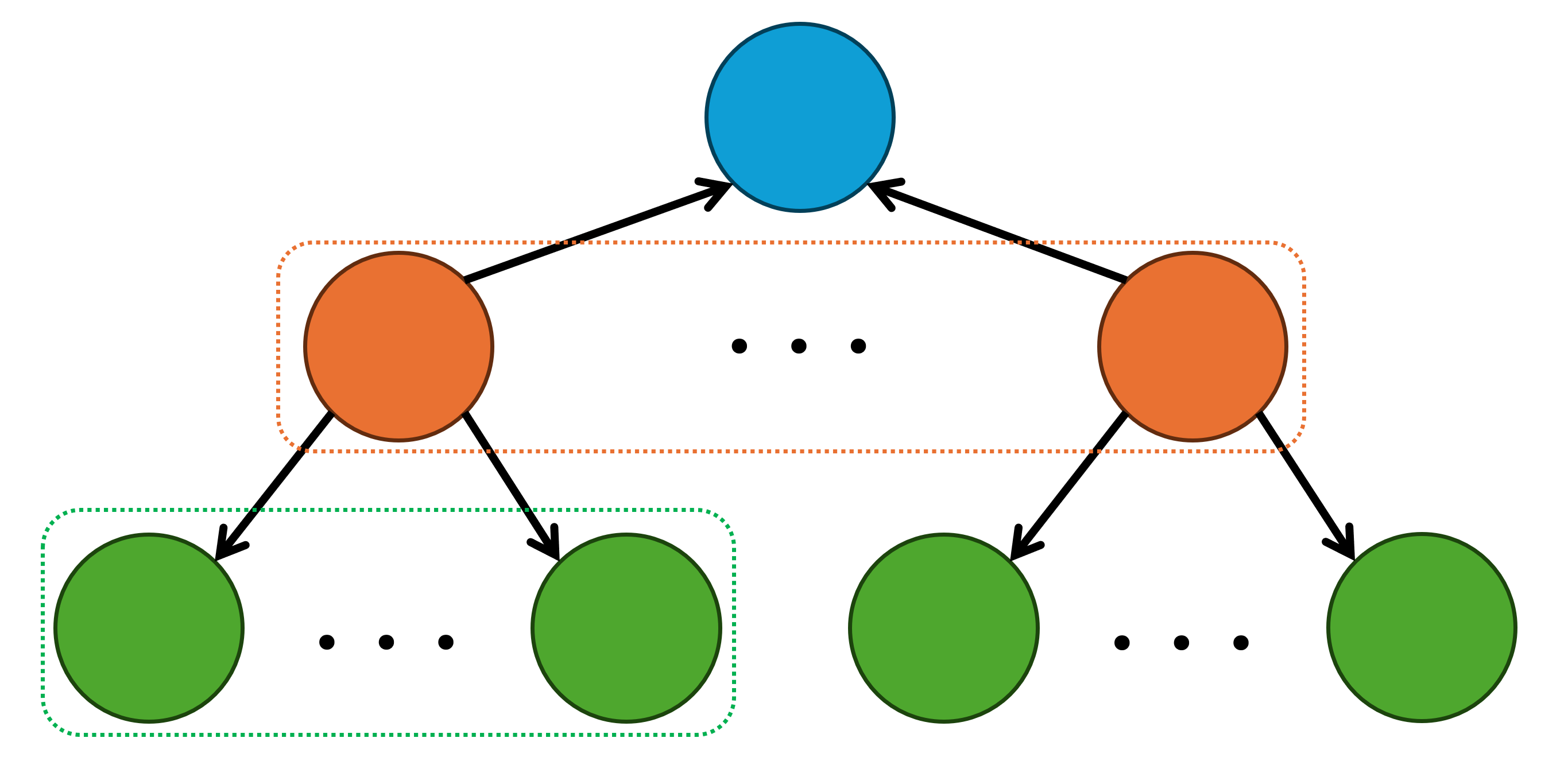}
    \put(50,39.5){\parbox{0.75\linewidth}\normalsize{\Large $i$}}
    \put(8,24.5){\parbox{0.75\linewidth}\normalsize{\Large $\mathcal{N}_i$}}
    \put(24,25){\parbox{0.75\linewidth}\normalsize{\Large $j$}}
    % \put(-7,7){\parbox{0.75\linewidth}\normalsize{\Large $\mathcal{N}_j$}}
    \put(7.5,7){\parbox{0.75\linewidth}\normalsize{\Large $k$}}
    \put(32,36){\parbox{0.75\linewidth}\normalsize{$\delta_{ji}$}}
    \put(10.5,18.5){\parbox{0.75\linewidth}\normalsize{$\delta_{jk}$}}
    \end{overpic}
    %\vspace{2ex}
    \caption{
        An illustration of Assumptions~\ref{assume:tree_structure} and \ref{assume:only_1hop_requests} for a given system, where, at the start of the first round of negotiation, requests are sent by the 1-hop neighborhood $\mathcal{N}_i$~(orange) of only one root agent $i \in [n]$~(blue), which are received by both $i$ and the 2-hop neighborhood~(green) of $i$.
    }
    \label{fig:tree_case_diagram}
\end{figure}
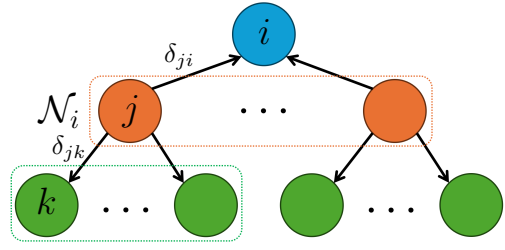

We establish the following lemmas regarding conditions for determining if $(\mathbf{x}_i, \mathbf{x}_{\mathcal{O}_i}), \forall i \in [n]$, is a terminally infeasible state under the above assumptions. 
%\vspace{1ex}
\begin{lemma} \label{lem:tau_geq2_term_infeas}
    Let Assumptions~\ref{assume:nonempty_convex_contraints}-\ref{assume:only_1hop_requests} hold.  
    For a given node $i \in [n]$, if $\exists j, \in \mathcal{N}_i, \, o \in \mathcal{O}_j$ such that $h_{j}(x_j, x_{o})$ is infeasible, or $\exists j, l \in \mathcal{N}_i, \, o \in \mathcal{O}_j, \, p \in \mathcal{O}_l$ such that $h_{j}(x_j, x_{o})$ and $h_{l}(x_l, x_{p})$ are jointly infeasible, for any collaborative round $\tau > 1$,  then $(\mathbf{x}_i, \mathbf{x}_{\mathcal{O}_i}), \forall i \in [n]$ is a terminally infeasible state. 
\end{lemma}
%\vspace{-3ex}
\begin{proof}
    By Algorithm~\ref{alg:colab_safety} and Assumptions~\ref{assume:syncronized} and \ref{assume:only_1hop_requests}, the system will enter a second round of collaboration (i.e., $\tau = 2$) if $\exists j \in \mathcal{N}_i \cup \{ i\}$ such that $\delta_j < 0$ by the end of the first round of collaboration. 
    By Algorithm~\ref{alg:colab_safety}, $\delta_j < 0$ implies that node~$j$ received an adjustment $\varepsilon_{kj}(\overline{u}_k) > 0, \forall k \in \mathcal{N}_j$, thereby necessitating an additional round of collaboration.
    
    By Assumptions~\ref{assume:tree_structure}-\ref{assume:only_1hop_requests} and Lemma~\ref{lem:max_ben_iff_full_cons}, if node~$k \in \mathcal{N}_j\setminus \{i\}$  sends an adjustment $\varepsilon_{kj}(\overline{u}_k) > 0$ to node~$j \in \mathcal{N}_i$, then node~$k$ will have selected a compromise-seeking action $\overline{u}_k \in \partial \mathcal{U}_k^s$ that is maximally beneficial to node~$j$ (since node~$k$ will have received requests from no other neighbors by Assumption~\ref{assume:only_1hop_requests}). Additionally, by convexity from Assumption~\ref{assume:nonempty_convex_contraints}, we have that $\max_{u_k^s \in \mathcal{U}_k^s}a_{jk,o}(\mathbf{x}_j, x_o)u_k^s$ is unique. 
    Thus, for all subsequent negotiation rounds, every node~$k \in \mathcal{N}_j$ will be fully constrained. 
    
    Note that if $\exists j \in \mathcal{N}_i, o \in \mathcal{O}_j$ such that $h_{j}(x_j, x_{o})$ is infeasible, then $\varepsilon_{ji}(\overline{u}_i) > 0$ also holds by Lemma~\ref{lem:infeas_implies_joint_infeas}.
    Therefore, if $\exists j, l \in \mathcal{N}_i, o \in \mathcal{O}_j, p \in \mathcal{O}_l$ such that $h_{j}(x_j, x_{o})$ and $h_{l}(x_l, x_{p})$ are jointly infeasible, then, by Algorithms~\ref{alg:coordinate} and \ref{alg:getClosestPoint}, node $i$ will send an adjustment $\varepsilon_{ji}(\overline{u}_i) > 0$ to node $j$ that must be reallocated to neighbors in $\mathcal{N}_j$. 
    However, since every node~$k \in \mathcal{N}_j \setminus \{i\}$ is fully constrained with respect to neighbor~$j$, every negotiation round will end with the same deficit $\delta_j < 0$ and node~$j$ will repeat the same requests $\delta_{ji}$ at each subsequent round by Line 20 of Algorithm~\ref{alg:coordinate}. 
    Thus, $\nexists u_i^s , u_{\mathcal{N}_i}^s \in \mathcal{U}_i^s \times \mathcal{U}_{\mathcal{N}_i}^s$ that will satisfy $\Phi_i(\mathbf{x}_i, \mathbf{x}_{\mathcal{O}_i}, u_i^s, u_{\mathcal{N}_i}^s) \geq 0$, making $(\mathbf{x}_i, \mathbf{x}_{\mathcal{O}_i})$ a terminally infeasible state for the system by Definition~\ref{def:terminally_infeasible}.
\end{proof}

\begin{lemma} \label{lem:at_least_one_constrained}
    Let Assumptions~\ref{assume:nonempty_convex_contraints}-\ref{assume:only_1hop_requests} hold and let the constraints $h_{j}(x_j, x_{o})$ for all $j \in \mathcal{N}_i$ and $o \in \mathcal{O}_j$ be weakly non-interfering. 
    % \edit{Further, let the initial state of the system satisfy Assumption~\ref{assume:only_1hop_requests} at $\tau = \upsilon = 1$}. 
    At any collaborative round $\tau>1$, if the system does not terminate with at least one safe action for each node, then node $i$ will become fully constrained with respect to at least one additional node $j \in \mathcal{N}_i$.
\end{lemma}
%\vspace{-3ex}
\begin{proof}
    If the constraints $h_{j}(x_j, x_{o_j})$ for all $j \in \mathcal{N}_i$ and $o \in \mathcal{O}_j$ are weakly non-interfering,
    then, from \cite[Theorem 3]{butler2023distributed} and by Assumption~\ref{assume:only_1hop_requests}, we can define the set of viable compromise-seeking actions at collaborative round $\tau \geq 1$ as 
    \begin{equation*}
        \begin{aligned}
            \overline{\underline{\partial \mathcal{U}}}_i^{\tau} &= \{u_i^s \in \partial \mathcal{U}_i^s: a_{ji,o} u_i^s  + \bar{c}_{ji}^\tau + \delta_{ji}^\tau  < 0 , \forall j \in \mathcal{N}_i\} \\
            & \quad \cap \{ u_i^s \in \partial \mathcal{U}_i^: a_{ji}^{\perp} u_i \geq 0, \forall j \in \mathcal{N}_i \},
        \end{aligned}
     \end{equation*}
    where $a_{ji}^{\perp}$ are the vectors orthogonal to $a_{ji}$ such that $a \cdot a_{ji}^{\perp} \geq 0$, with $a \in \mathbb{R}^{M_i}$ being any vector that satisfies the property of weakly non-interfering for $h_{j}(x_j, x_{o_j})$ for all $j \in \mathcal{N}_i$ and $o \in \mathcal{O}_j$.
    
    By \cite[Theorem 3]{butler2023distributed}, $\overline{\underline{\partial \mathcal{U}}}_i^{\tau}$
    is contracting for jointly infeasible neighbors at every round $\tau > 1$ since $\delta_{ji}^{\tau} \leq \delta_{ji}^{\tau + 1} \leq 0$.
    Therefore, if node $i$ is fully constrained with respect to any neighbor $j \in \mathcal{N}_i$, it will remain fully constrained for all subsequent rounds of negotiation. 
    
    Further, by Algorithm~\ref{alg:colab_safety} and Assumptions~\ref{assume:syncronized} and \ref{assume:only_1hop_requests}, at any round of negotiation where the collaborative round is $\tau>1$, node $i$ will choose a compromise-seeking action $\overline{u}_i \in \partial \mathcal{U}_i$ according to Algorithm~\ref{alg:getClosestPoint}, where by Lemma~\ref{lem:tau_geq2_term_infeas} all requests must be feasible, otherwise $(\mathbf{x}_i, \mathbf{x}_{\mathcal{O}_i}), \forall i \in [n]$, is a terminally infeasible state. Thus, node $i$ will comply fully with the current feasible requests and adjust non-requesting nodes $j \in \mathcal{N}_i$ by $\varepsilon_{ji}(\overline{u}_i)$, if $\varepsilon_{ji}(\overline{u}_i) > 0$. All nodes $j \in \mathcal{N}_i$ will then attempt to allocate the adjustment $\varepsilon_{ji}(\overline{u}_i)$ to neighbors $k \in \mathcal{N}_j \setminus \{ i\}$ by making requests $\delta_{kj}$, where $\delta_{j} = \sum_{k \in \mathcal{N}_j} \delta_{kj} = -\varepsilon_{ji}(\overline{u}_i)$. 
    
    Therefore, for the system to enter round $\tau+1$, at least one additional neighbor $j \in \mathcal{N}_i$ must have received adjustments $\varepsilon_{jk}(\overline{u}_k) > 0, \forall k \in \mathcal{N}_j$ such that $\delta_j < 0$, which by the proof of Lemma~\ref{lem:tau_geq2_term_infeas}, through Assumptions~\ref{assume:tree_structure}-\ref{assume:only_1hop_requests} and Lemma~\ref{lem:max_ben_iff_full_cons}, implies that every node $k \in \mathcal{N}_j$ is fully constrained with respect to neighbor $j$
    % (i.e., $\overline{u}_k$ will be unchanging $\forall k \in \mathcal{N}_j$) 
    and the system will start a new round of collaboration.  
    Thus, if the system does not terminate with at least one safe action for each node, node $i$ will become fully constrained with respect to at least one additional node $j \in \mathcal{N}_i$ at each collaborative round.
\end{proof}

\subsection{Linear-Time Convergence Result}
We now show that there is an upper bound on the number of collaborative rounds needed to find at least one safe action for all neighbors if it exists. 
%\vspace{1ex}
\begin{theorem} \label{thm:finite_time_converge_tree_net}
    Let Assumptions~\ref{assume:nonempty_convex_contraints}-\ref{assume:only_1hop_requests} hold and let the constraints $h_{j}(x_j, x_{o})$ for all $j \in \mathcal{N}_i$ and $o \in \mathcal{O}_j$ be weakly non-interfering. If $\tau > |\mathcal{N}_i|$, then $(\mathbf{x}_i, \mathbf{x}_{\mathcal{O}_i}), \forall i \in [n]$, is a terminally infeasible state.
\end{theorem}
%\vspace{-3ex}
\begin{proof}
In this proof, we detail cases for each collaborative round of the system and show that if $\tau > |\mathcal{N}_i|$, then $(\mathbf{x}_i, \mathbf{x}_{\mathcal{O}_i}), \forall i \in [n]$, must be a terminally infeasible state.

\textbf{\underline{Round 1:}} According to Assumption~\ref{assume:only_1hop_requests}, at $\tau = \upsilon = 1$ node $i$ will receive at most $|\mathcal{N}_i|$ requests $(A_{ji,\mathcal{O}_j}, \delta_{ji})$, where $A_{ji,\mathcal{O}_j} \in \mathbb{R}^{K_j \times M_i}$ and $\delta_{ji} \in \mathbb{R}^{K_j}$, where $K_j \leq |\mathcal{O}_j|$ is the number of obstacles agent~$j$ is requesting assistance from agent~$i$ to avoid, with $K_j = 1$ by Assumption~\ref{assume:only_1hop_requests}.
If $\exists j,l \in \mathcal{N}_i, \, o \in \mathcal{O}_j, \, p \in \mathcal{O}_l$ such that $h_{j}(x_j, x_{o})$ and $h_{l}(x_l, x_{p})$ are jointly infeasible for node $i$, then the root node sends adjustments $\varepsilon_{ji}$ to all jointly infeasible 1-hop neighbors. 

All 1-hop neighbors $j \in \mathcal{N}_i$ then make secondary requests $\delta_{jk}$ to 2-hop neighbors $k \in \mathcal{N}_j \setminus \{ i\}$ in subsequent rounds of negotiation to attempt to reallocate the adjustment sent by the root node $i$. If no 2-hop neighbors send adjustments $\varepsilon_{jk}>0$ in response, then the algorithm terminates with a safe control action for all nodes. Otherwise, if there exists a 1-hop neighbor where all 2-hop neighbors send adjustments, then the algorithm enters a second round of collaboration.

\textbf{\underline{Round 2:}}
At collaboration round $\tau = 2$, all 1-hop neighbors that were unable to reallocate their safety needs to 2-hop neighbors during the first round make another request to the root node. 

\textbf{Case 1:} If there exists a single infeasible request, or jointly infeasible requests, from 1-hop neighbors in the second round, then $(\mathbf{x}_i, \mathbf{x}_{\mathcal{O}_i}), \forall i \in [n]$, is a terminally infeasible state by Assumptions~\ref{assume:nonempty_convex_contraints}-\ref{assume:only_1hop_requests} and Lemma~\ref{lem:tau_geq2_term_infeas}.

\textbf{Case 2:} If the requests are feasible, the root node follows Algorithm~\ref{alg:getClosestPoint} and sends adjustments $\varepsilon_{ji}$ to non-requesting 1-hop neighbors from this round. Then, 1-hop neighbors $j \in \mathcal{N}_i$ that receive adjustments from the root make additional requests $\delta_{jk}$ of the 2-hop neighbors.

\textbf{Case 2.1:} If each 1-hop neighbor can successfully allocate the adjustment $\varepsilon_{ji}$ from the root node to 2-hop neighbors $k \in \mathcal{N}_j$, then the system terminates with a safe action for all nodes.

\textbf{Case 2.2:} If there exists a 1-hop neighbor where all 2-hop neighbors send back adjustments $\varepsilon_{jk} > 0, \forall k \in \mathcal{N}_j$, then the algorithm enters another round of collaboration with the same series of cases as detailed in Round 2.

For $\tau > 2$, by Assumptions~\ref{assume:nonempty_convex_contraints}-\ref{assume:only_1hop_requests} and Lemma~\ref{lem:at_least_one_constrained}, if the constraints $h_{j}(x_j, x_{o_j})$, for all $j \in \mathcal{N}_i$ and $o_j \in \mathcal{O}_j$, are weakly non-interfering, then node $i$ will be fully constrained for at least one additional 1-hop neighbor at each round. Thus, if the algorithm has not terminated with a safe action for all nodes by round $\tau = |\mathcal{N}_i|$, then the set of neighbors $\mathcal{N}_j, \forall j \in \mathcal{N}_i$ must be fully constrained with respect to each agent~$j$ and $(\mathbf{x}_i, \mathbf{x}_{\mathcal{O}_i}), \forall i \in [n]$ is a terminally infeasible state.
\end{proof}

This analysis illustrates some of the difficulties in analyzing the finite-time convergence of Algorithm~\ref{alg:colab_safety}. Note that if we expand our analysis by relaxing Assumption~\ref{assume:only_1hop_requests} to include the resolution of multiple requests from multiple neighborhoods, then Lemma~\ref{lem:tau_geq2_term_infeas} would no longer hold since multiple requests at both the 1-hop and 2-hop neighborhood levels may incur compromises of their own that must be resolved in rounds $\tau > 1$. We conjecture that determining the worst-case convergence rate of Algorithm~\ref{alg:colab_safety} will continue to be a function of the degree of each node (i.e. $|\mathcal{N}_i|$), but this analysis is left for future work.
Note, however, that this result not only shows a sufficient termination condition of Algorithm~\ref{alg:colab_safety}, but also sheds light on potential ways of simplifying networks for better performance of distributed algorithms.

\section{Mass-Spring Formation Dynamics} \label{sec:formation_dynamics}

We now illustrate the application of our collaborative safety algorithm to safe cooperative formation control of a simplified two-dimensional agent system and simulate a multi-obstacle avoidance scenario.
\subsection{Virtual Mass-Spring Formation Model}
Consider a two-dimensional multi-agent system with distributed formation control dynamics defined by a virtual mass-spring model, with $x_i = [p_i^{\vec{x}}, p_i^{\vec{y}}, v_i^{\vec{x}}, v_i^{\vec{y}}]^\top$
\begin{equation} \label{eq:formation_dynamics_mass_spring}
    \dot{x}_i =
    \begin{bmatrix}
        v_i^{\vec{x}} \\ v_i^{\vec{y}} \\ 0 \\ 0
    \end{bmatrix}
    +
    \begin{bmatrix}
        0 & 0 \\
        0 & 0 \\
        1 & 0 \\
        0 & 1 
    \end{bmatrix}
    \left( u_i^f(x) - u_i^s \right)
\end{equation}
where
\begin{equation} \label{eq:formation_controller_mass_spring}
   u_i^f(x) =
   \begin{bmatrix}
    u_i^{f_{\vec{x}}} \\
    u_i^{f_{\vec{y}}}
   \end{bmatrix}
   =
   \frac{1}{m_i}
   \begin{bmatrix}
    -\sum_{j \in \mathcal{N}_i} k_{ij} s_{ij} \sin \theta_{ij} - b_{ij}v_i^{\vec{x}} \\
    -\sum_{j \in \mathcal{N}_i} k_{ij} s_{ij} \cos \theta_{ij} - b_{ij}v_i^{\vec{y}} \\
   \end{bmatrix}
\end{equation}
describes the desired formation behavior of the system, where agents behave as if coupled by mass-less springs with $k_{ij}$ and $b_{ij}$ being the spring and dampening constants for the virtual spring from agent~$j$ to agent~$i$, respectively, and
\begin{equation*}
    s_{ij} = L_{ij} - R_{ij}
\end{equation*}
denoting the stretch length of a given spring connection with resting length $R_{ij}$ and
\begin{equation*}
    L_{ij} = \Vert p_i - p_j\Vert_2
\end{equation*}
being the current length of the spring. We compute the $\vec{x}$ and $\vec{y}$ components of the stretched spring as
\begin{equation*}
    \sin \theta_{ij} = \frac{p_i^{\vec{x}} - p_j^{\vec{x}}}{L_{ij}}, \;\; \cos \theta_{ij} = \frac{p_i^{\vec{y}} - p_j^{\vec{y}}}{L_{ij}}.
\end{equation*}
%\vspace{-3ex}

Thus, our induced coupling model then becomes
\begin{equation}
    \bar{f}_i(x) =
    \begin{bmatrix}
        v_i^{\vec{x}} \\ v_i^{\vec{y}} \\ u_i^{f_{\vec{x}}} \\ u_i^{f_{\vec{y}}}
    \end{bmatrix}
    , \;\;
    \bar{g}_i = 
    \begin{bmatrix}
        0 & 0 \\
        0 & 0 \\
        -1 & 0 \\
        0 & -1 
    \end{bmatrix}
\end{equation}
with the first-order safety condition for a given obstacle using the barrier function candidate \eqref{eq:rel_dist_barrierfunc} computed as
\begin{equation}
    \begin{aligned}
        \phi_{i,o}^1(x_i, x_o) &= 2\left[v_i^{\vec{x}}(p_i^{\vec{x}} - p_o^{\vec{x}}) + v_i^{\vec{y}}(p_i^{\vec{y}} - p_o^{\vec{y}}) \right] \\ 
        & \quad + \alpha_i^0(h_{i,o}(x_i, x_o))
    \end{aligned}
\end{equation}
which yields the Lie derivatives of the safety condition with respect to the formation dynamics as
\begin{equation}
    \begin{aligned}
        \mathcal{L}_{\bar{f}_i}\phi_{i,o}^1(\mathbf{x}_i, x_o) &= 2 v_i^{\vec{x}}(v_i^{\vec{x}} + \alpha_i^0(p_i^{\vec{x}}-p_o^{\vec{x}})) \\ 
        & \quad + v_i^{\vec{y}}(v_i^{\vec{y}} + \alpha_i^0(p_i^{\vec{y}}-p_o^{\vec{y}})) \\
        &\quad + u_i^{f_{\vec{x}}}(p_i^{\vec{x}} - p_o^{\vec{x}}) + u_i^{f_{\vec{y}}}(p_i^{\vec{y}} - p_o^{\vec{y}}) 
    \end{aligned}
\end{equation}
and
\begin{equation}
    \mathcal{L}_{\bar{g}_i}\phi_{i,o}^1(x_i, x_o) = 2 
    \begin{bmatrix}
        p_i^{\vec{x}} - p_o^{\vec{x}} & p_i^{\vec{y}} - p_o^{\vec{y}}
    \end{bmatrix}.
\end{equation}

It should be noted that given this mass-spring network formation control law, when computing the effect of control by agent~$j$ on the safety conditions of agent~$i$
% , assuming both agents apply control through acceleration, 
yields
\begin{equation}
    \mathcal{L}_{\bar{g}_j} \mathcal{L}_{\bar{f}_i} \phi_{i,o}^1(x_i, x_o) = \mathbf{0}^{M_j}
\end{equation}
since the control input of agent~$j$ does not appear until the next derivative of $\Phi_i$. In order to avoid unnecessary computations of additional partial derivatives, each agent computes the effect of neighboring control as if neighbors directly control their velocities, i.e.,
\begin{equation}
    \bar{g}_j = 
    \begin{bmatrix}
        -1 & 0 \\
        0 & -1 \\
        0 & 0 \\
        0 & 0 
    \end{bmatrix}, \forall j \in \mathcal{N}_i.
\end{equation}
This assumption is non-physical since it would require infinite acceleration for neighbors to achieve such a discontinuous instantaneous jump in velocity. However, if we assume a finite time interval $\tau >0$ during which our acceleration controller might achieve such a change in velocity, we can approximate the necessary acceleration constraints during that time. In other words, since these terms are used to communicate action limitations on neighbors, we may approximate acceleration limits over a given time interval by simply dividing the velocity constraints by the appropriate time window length. Therefore, if the velocity constraints communicated are
\begin{equation}
    \mathcal{U}^v = \{ u \in \mathbb{R}^M: Gu +l \leq \mathbf{0} \}
\end{equation}
then we may compute acceleration constraints for a given time interval $\tau > 0$ as
\begin{equation}
    \mathcal{U}^a = \left\{ u \in \mathbb{R}^M: \frac{1}{\tau} Gu +l \leq \mathbf{0} \right\}.
\end{equation}
In application, this reliance on a known time interval may cause challenges when accounting for communication delays and inconsistent processing and actuation time intervals.

\subsection{Simulations} \label{sec:simulations}

In this section, we demonstrate the performance of our collaborative safety algorithm in the special case of a tree-structured formation network avoiding one obstacle and show that collaborative rounds do not exceed the number of neighbors for a given agent. We then demonstrate the performance of the algorithm in the unproven general case with a fully-connected formation network with multiple obstacles, both static and dynamic. To view our simulation code, see~\cite{butlerRepo2023}.

\subsubsection{Tree Network with Single Obstacles}
We construct an example for a tree-structured networked with 7 agents where the parameters of the virtual mass-spring system are $m_i = 0.5$, $r_{i,o} = 1$, $K_{ij}=3$, $R_{ij} = 3$, and $b_{ij} = 1$ for all $i \in [n]$, $j \in \mathcal{N}_j$. We design the network structure such that each agent has at most 2 children in the tree network. Further, we set control magnitude limits for each agent~$i$ as $\mathcal{U}_i = \{ u_i \in \mathbb{R}^2: \Vert u_i \Vert_{\infty} \leq 20 \}$. We then apply a constant control signal to each agent which leads the formation towards a single obstacle, where the initial and final positions of each agent and their respective trajectories through the obstacle field are shown in Figure~\ref{fig:obs_trajectory_7a}. Each agent uses the modified collaborative safety algorithm described in Algorithm~\ref{alg:colab_safety} to communicate its safety needs and accommodate safety requests to and from neighbors, respectively. Each agent then implements a first-order safety filter on their control actions as described by \eqref{eq:problem_statement} while incorporating the control constraints $\overline{\mathcal{U}}_i^s$ computed using Algorithm~\ref{alg:colab_safety}. We plot the safety-filtered control signal including the constant control signal for each agent in Figure~\ref{fig:control_safety_filter_7a}, which shows the safety-filtered control signal $u_i^s$ for both the $\vec{x}$ and $\vec{y}$ components over time. 

Further, we also see in Figure~\ref{fig:obs_trajectory_7a} that the number of collaborative rounds never exceeds the number of children for each agent, consistent with the results of Theorem~\ref{thm:finite_time_converge_tree_net}. For a video of this simulation, see \href{https://youtube.com/watch/TwnFhhScSOk}{https://youtube.com/watch/TwnFhhScSOk}. 
Note that since the spring network is not fully collected, the system allows for inter-agent collision (as seen by the final position of the agents); however, we can prevent inter-agent collisions by either incorporating other agents as obstacles to be avoided in the CBF formulation, or use a fully connected formation network which incorporates inter-agent collision avoidance naturally in the spring dynamics, as illustrated in the following simulation.  

\begin{figure}
    \centering
    \includegraphics[width=.8\columnwidth]{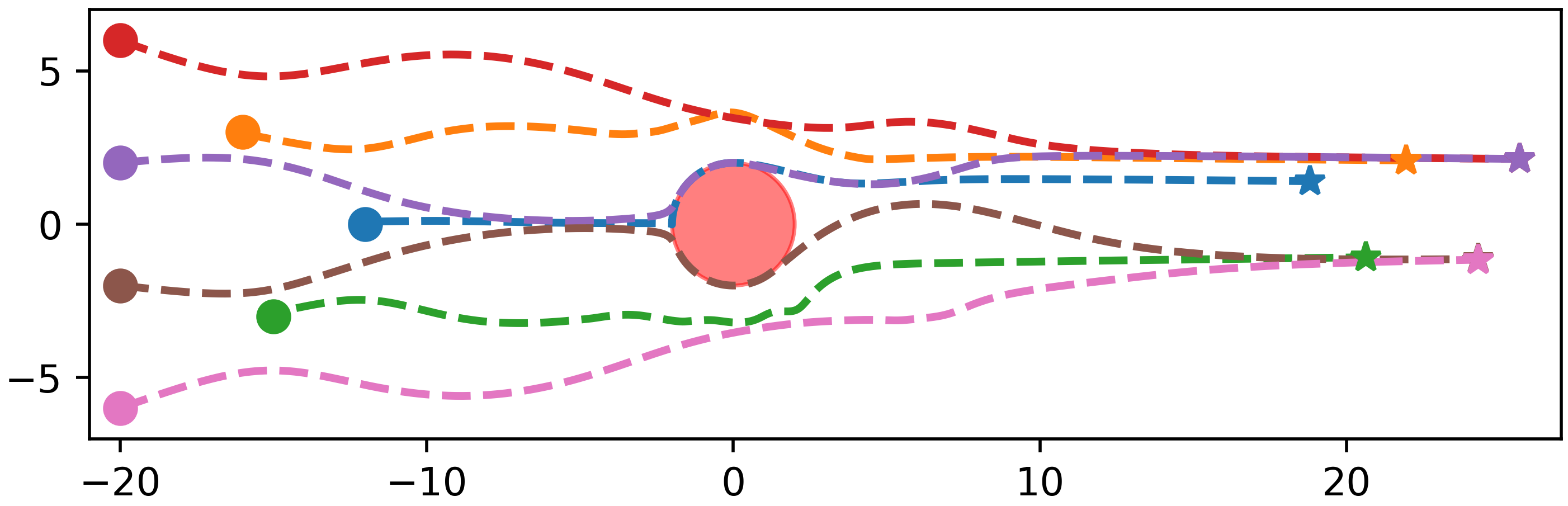}
    \caption{The trajectories of a 7-agent tree-structure formation network avoiding a single obstacle, where each agent is given a constant control signal directing it in the positive $x$ direction. Each agent implements safety filtering according to Algorithm~\ref{alg:colab_safety} and \eqref{eq:problem_statement} to avoid the obstacle while maintaining a formation behavior, according to \eqref{eq:formation_dynamics_mass_spring} and \eqref{eq:formation_controller_mass_spring}.}
    \label{fig:obs_trajectory_7a}
\end{figure}

\begin{figure}
    \centering
    \begin{subfigure}[b]{0.35\textwidth}
         \centering
         \begin{overpic}[width=\textwidth]{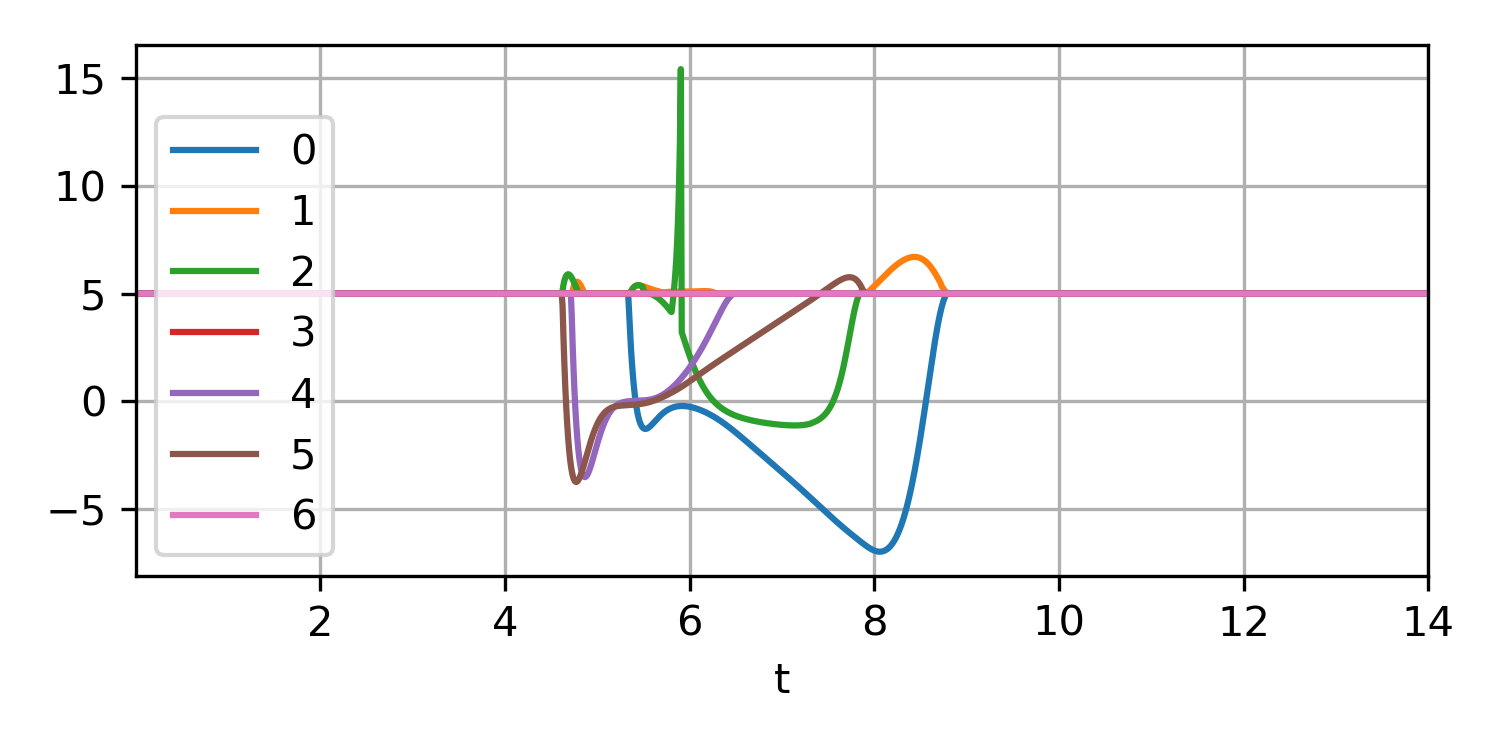}
             \put(-6.5,26){\parbox{0.75\linewidth}\normalsize \rotatebox{90}{$u_i^{s,\vec{x}}$}}
         \end{overpic} 
    \end{subfigure}
    \begin{subfigure}[b]{0.35\textwidth}
         \centering
         \begin{overpic}[width=\textwidth]{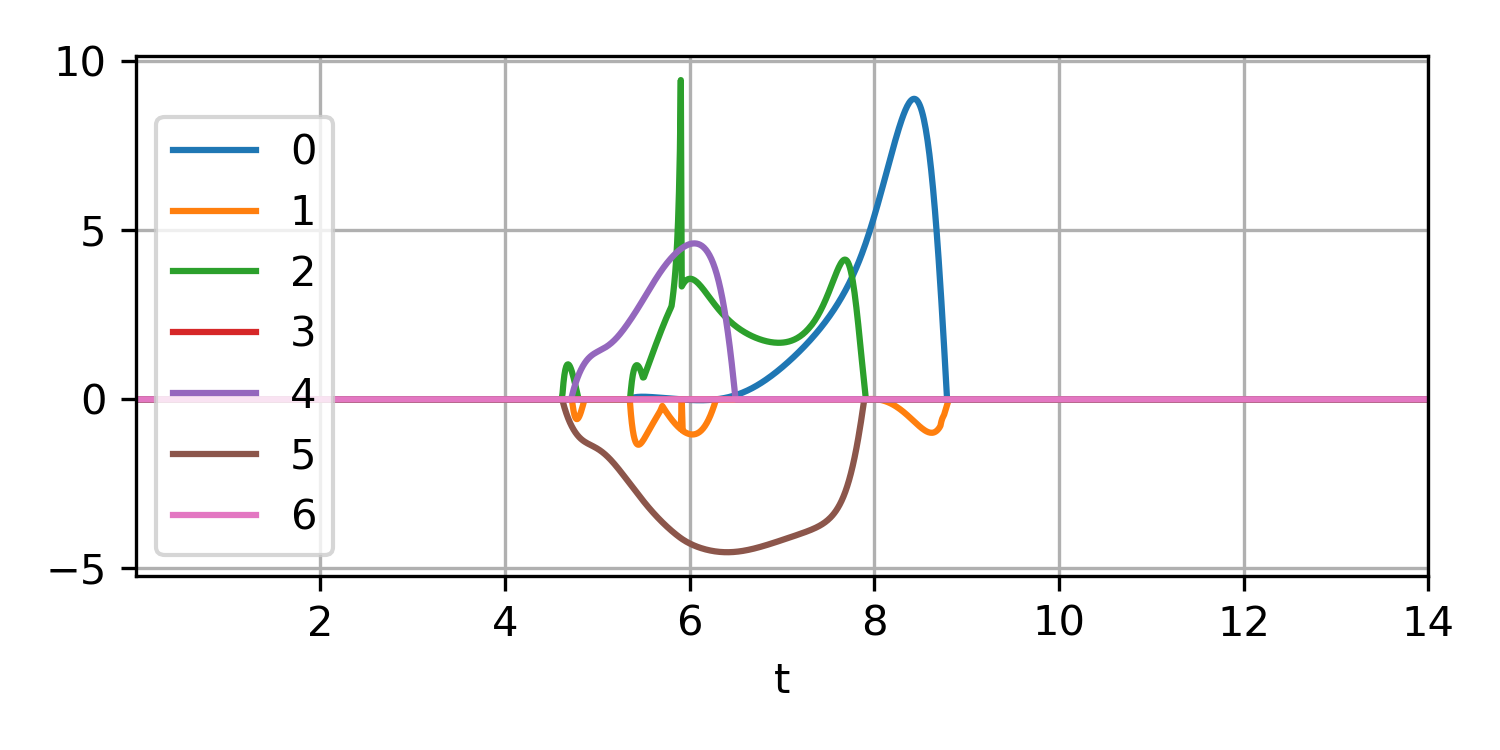}
             \put(-6.5,26){\parbox{0.75\linewidth}\normalsize \rotatebox{90}{$u_i^{s,\vec{y}}$}}
         \end{overpic}
    \end{subfigure}
    \begin{subfigure}[b]{0.35\textwidth}
         \centering
         \begin{overpic}[width=\textwidth]{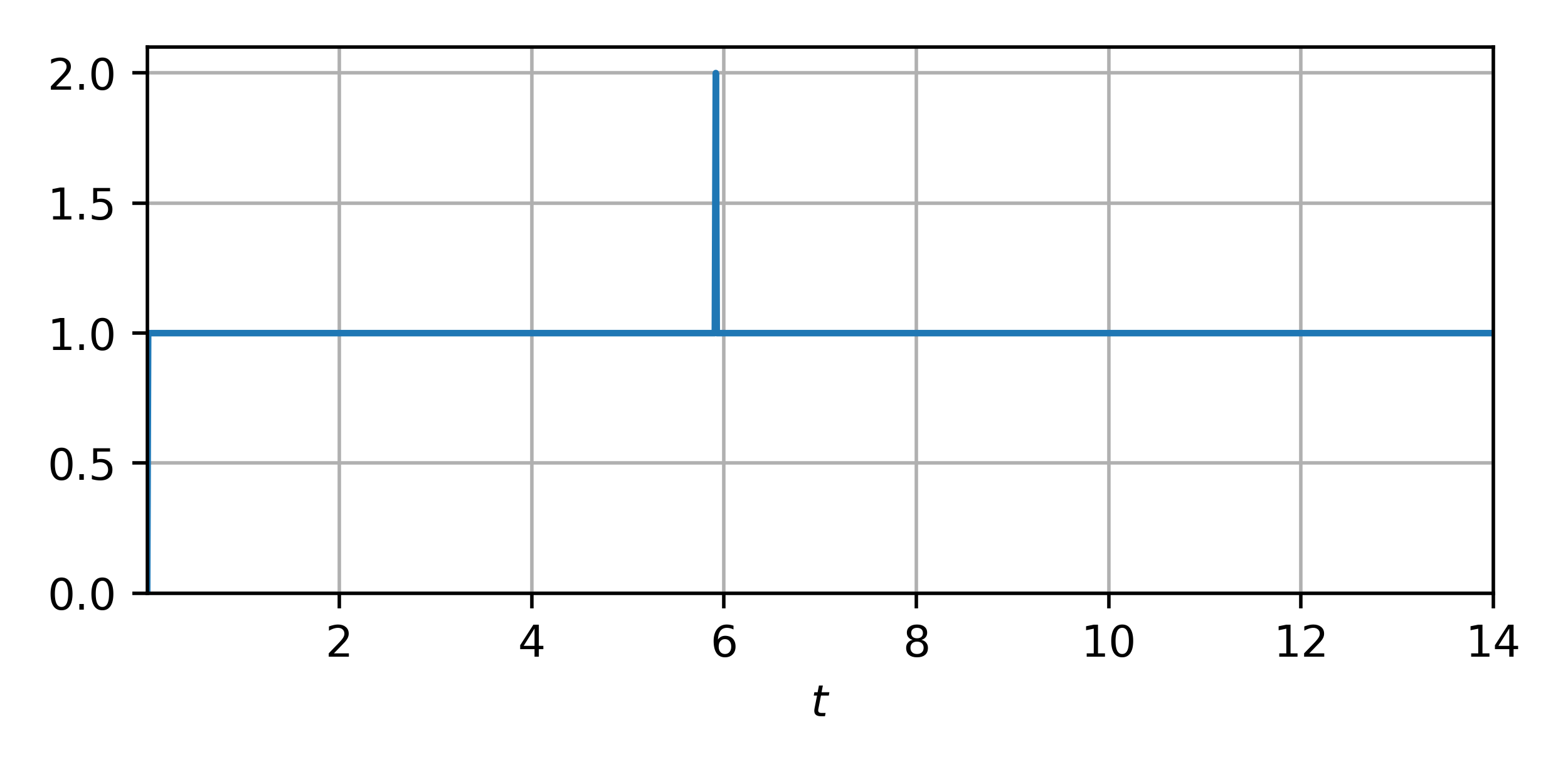}
             \put(-4,29){\parbox{0.75\linewidth}\normalsize \rotatebox{90}{$\tau$}}
         \end{overpic}
    \end{subfigure}
    \caption{The safety-filtered control signals for each agent in the $\vec{x}$ component (top) and $\vec{y}$ component (middle) of $u_i^s$, which are computed using Algorithm~\ref{alg:colab_safety} and \eqref{eq:problem_statement}, during the traversal of the formation around the obstacle shown in Figure~\ref{fig:obs_trajectory_7a}. We also show the number of collaborative rounds $\tau$ (bottom) at each time-step of the simulation, where the sampling rate is 100 Hz.}
    \label{fig:control_safety_filter_7a}
\end{figure}

\subsubsection{Fully Connected Networked with Multiple Obstacles}
We test our collaborative safety algorithm in the case of a fully connected formation of 8 agents with the same parameters as the tree network case, with the exception that the control magnitude limits for each agent~$i$ are $\mathcal{U}_i = \{ u_i \in \mathbb{R}^2: \Vert u_i \Vert_{\infty} \leq 20 \}$. We similarly apply a constant signal to each agent that leads the formation through an obstacle field as shown in Figure~\ref{fig:obs_trajectory_8a}, with the safety-filtered signals over time being shown in Figure~\ref{fig:control_safety_filter_8a}. Note in Figure~\ref{fig:control_safety_filter_8a} that even in the fully-connected case with multiple obstacles the network returns a safe action for all agents at a rate that is rapid enough for limited communication (a maximum of 13 rounds of communication). For a video of this simulation, see \href{https://youtube.com/watch/XZLdFNK7MfE}{https://youtube.com/watch/XZLdFNK7MfE}. We also use this same case to test obstacle avoidance in a dynamic environment where the formation successfully avoids the moving obstacles despite obstacle velocity not being incorporated into the design of the safety conditions, which can be viewed at \href{https://youtube.com/watch/HAUGE882PlE}{https://youtube.com/shorts/HAUGE882PlE}. However, once the obstacles' or agents' speeds are increased significantly, we find that the system reaches an unsafe state, as each agent is unable to react quickly enough to the changing environment.

\begin{figure}
    \centering
    \includegraphics[width=.8\columnwidth]{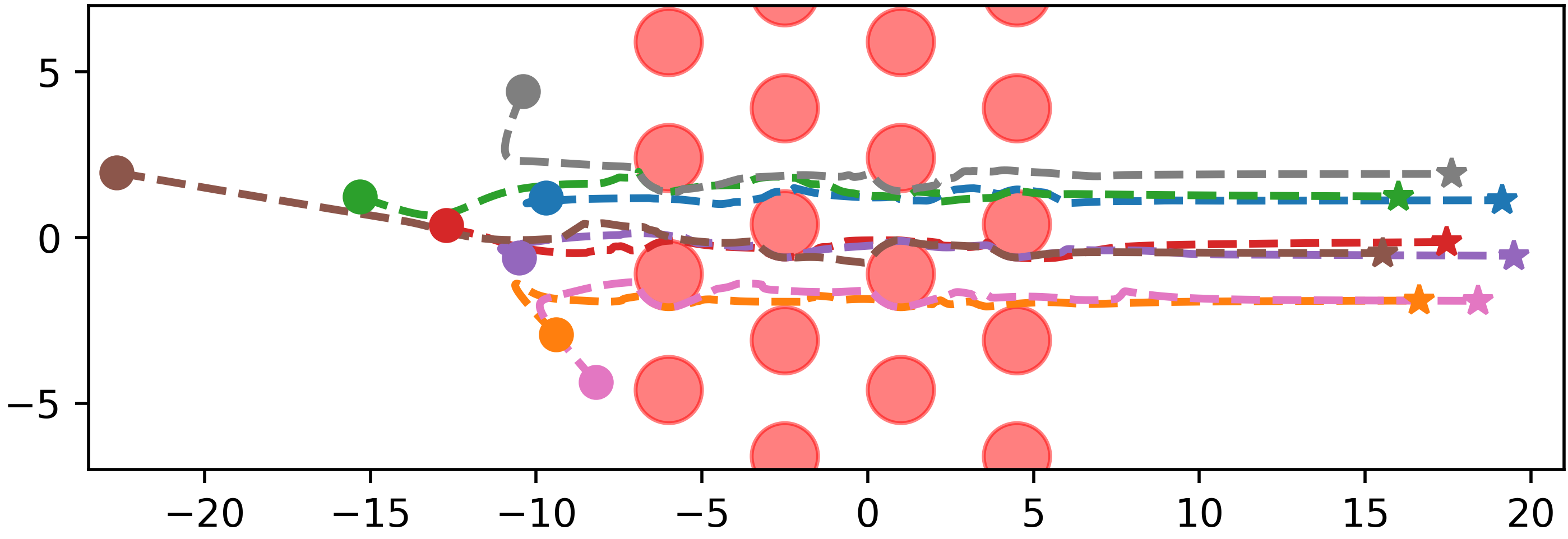}
    \caption{The trajectories of an 8-agent fully-connected formation network avoiding an obstacle field, where each agent is given a constant control signal directing it in the positive $x$ direction. Each agent implements safety filtering according to Algorithm~\ref{alg:colab_safety} and \eqref{eq:problem_statement} to avoid obstacles while maintaining a formation behavior, according to \eqref{eq:formation_dynamics_mass_spring} and \eqref{eq:formation_controller_mass_spring}.}
    \label{fig:obs_trajectory_8a}
\end{figure}

\begin{figure}
    \centering
    \begin{subfigure}[b]{0.35\textwidth}
         \centering
         \begin{overpic}[width=\textwidth]{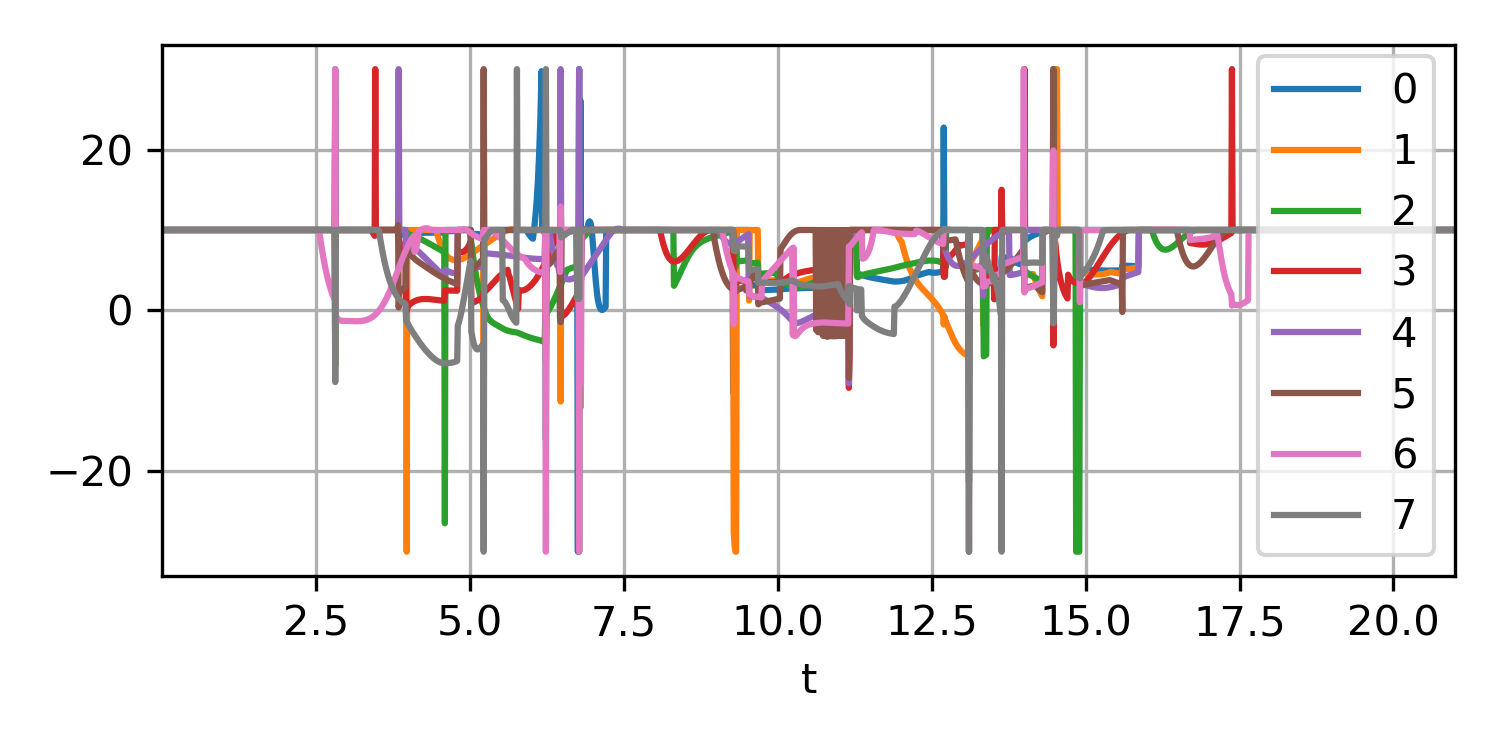}
             \put(-6,26){\parbox{0.75\linewidth}\normalsize \rotatebox{90}{$u_i^{s,\vec{x}}$}}
         \end{overpic} 
    \end{subfigure}
    \begin{subfigure}[b]{0.35\textwidth}
         \centering
         \begin{overpic}[width=\textwidth]{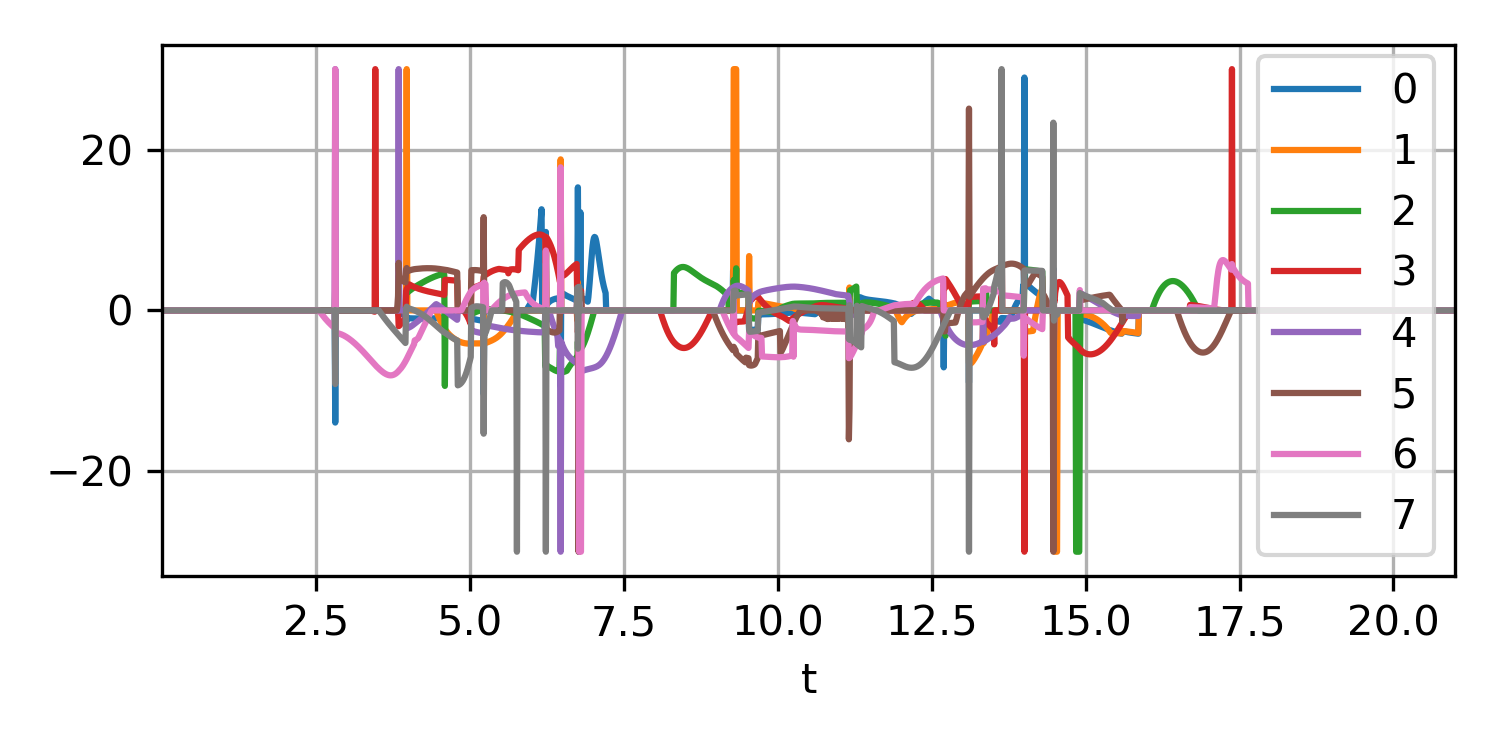}
             \put(-5,26){\parbox{0.75\linewidth}\normalsize \rotatebox{90}{$u_i^{s,\vec{y}}$}}
         \end{overpic}
    \end{subfigure}
    \begin{subfigure}[b]{0.35\textwidth}
         \centering
         \begin{overpic}[width=\textwidth]{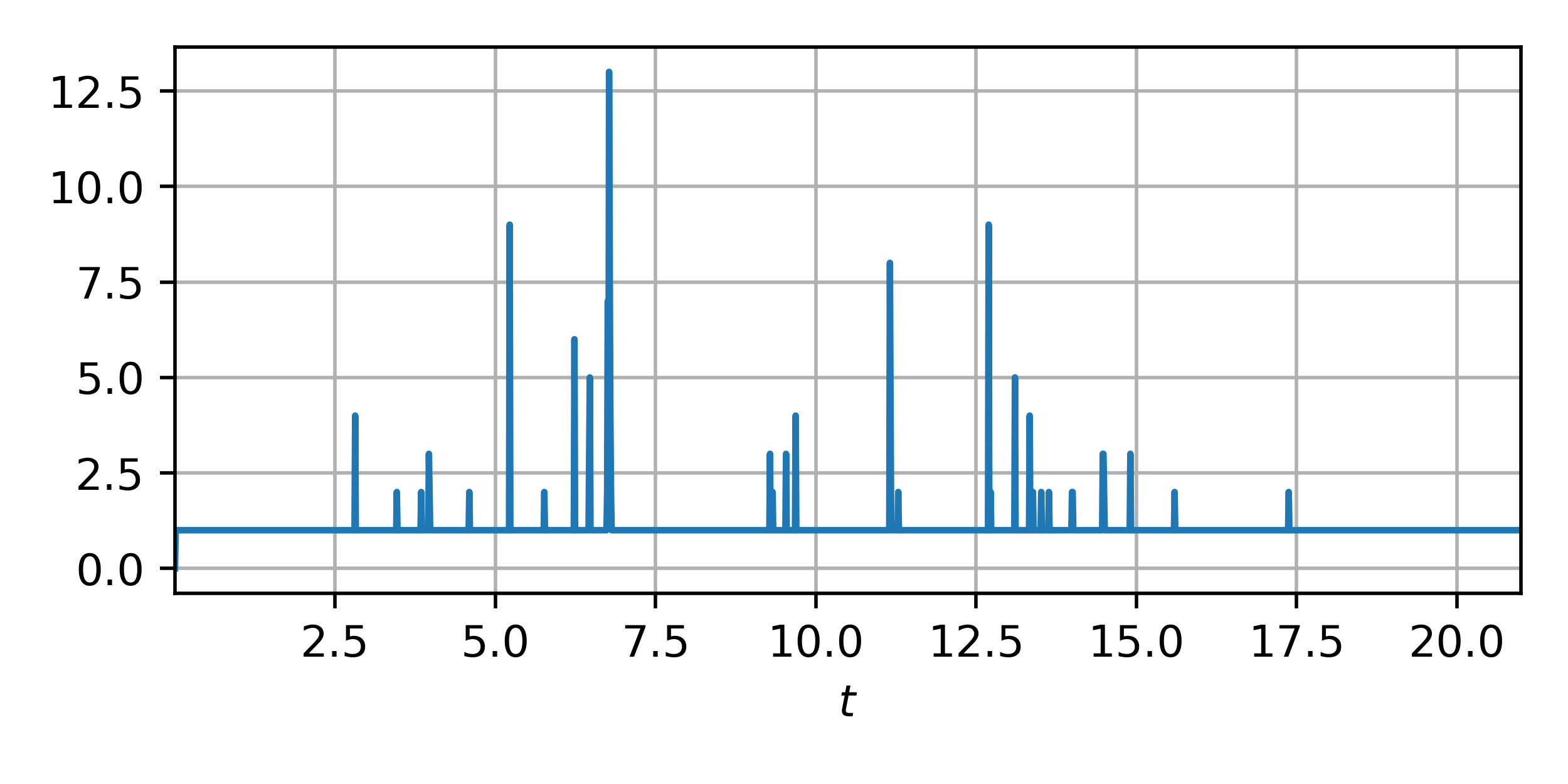}
             \put(-4,29){\parbox{0.75\linewidth}\normalsize \rotatebox{90}{$\tau$}}
         \end{overpic}
    \end{subfigure}
    \caption{The safety-filtered control signals for each agent in the $\vec{x}$ component (top) and $\vec{y}$ component (middle) of $u_i^s$, which are computed using Algorithm~\ref{alg:colab_safety} and \eqref{eq:problem_statement}, during the traversal of the formation through the obstacle field shown in Figure~\ref{fig:obs_trajectory_8a}. We also show the number of collaborative rounds $\tau$ (bottom) at each time-step of the simulation, where the sampling rate is 100 Hz.}
    \label{fig:control_safety_filter_8a}
\end{figure}

\section{Conclusion}\label{sec:conclusion}
In this work, we have presented a method for implementing a collaborative safety algorithm on formation control problems. We utilize barrier function methods to describe the network effects on virtually coupled agents via high-order barrier functions, which provide a way to encode the relative effects of neighbors' actions on each other's safety conditions. We provided an initial analysis of the worst-case convergence rate of our algorithm to terminate with a safe action in the special case of a tree-structured communication network for a single obstacle. Important areas of future work include extending our analysis to handle the general case of multiple requests in a generic network structure, as well as testing the viability of such a communication scheme to be carried out by real agents with constrained communication speed and bandwidth.

\bibliographystyle{IEEEtran}        % Include this if you use bibtex 
\bibliography{references} 

\end{document}